\documentclass{article}

\usepackage[main, final]{neurips_2026}

\usepackage[utf8]{inputenc} 
\usepackage[T1]{fontenc}    
\usepackage{hyperref}       
\usepackage{booktabs}       
\usepackage{amsfonts}       
\usepackage{nicefrac}       
\usepackage{microtype}      
\usepackage{array}
\usepackage{graphicx}
\usepackage{subcaption}
\usepackage[dvipsnames,table]{xcolor}

\usepackage{amsmath}
\usepackage{amssymb}
\usepackage{mathtools}
\usepackage{amsthm}
\usepackage{float}

\usepackage[capitalize,noabbrev]{cleveref}

\theoremstyle{plain}
\newtheorem{theorem}{Theorem}[section]

\newtheorem{lemma}[theorem]{Lemma}

\theoremstyle{definition}

\theoremstyle{remark}

\DeclareMathOperator*{\argmax}{argmax}

\newcommand{\JointSLBO}{\texttt{JoBS}}

\DeclarePairedDelimiter\floor{\lfloor}{\rfloor}

\newcommand{\squishlisttwo}{
 \begin{list}{$\bullet$}
  { \setlength{\itemsep}{1pt}
     \setlength{\parsep}{0pt}
    \setlength{\topsep}{0pt}
    \setlength{\partopsep}{0pt}
    \setlength{\leftmargin}{1em}
    \setlength{\labelwidth}{1.5em}
    \setlength{\labelsep}{0.5em} } }

\newcommand{\squishend}{
  \end{list}  }
\usepackage{placeins}
\newcommand{\Bsmall}{B_{\text{small}}}
\definecolor{darkgreen}{rgb}{0.0, 0.5, 0.0}
\usepackage{multirow}
\usepackage{anyfontsize}
\usepackage{bm}
\usepackage{wrapfig}
\usepackage{thm-restate}
\usepackage[textsize=tiny]{todonotes}

\title{Rethinking Bayesian Optimization for Co-optimizing LLM Training Configurations}

\author{Zhiliang Chen$^{*,1,2}$\hspace{0.5mm}, Alfred Wei Lun Leong$^{*,1}$, Shao Yong Ong$^{*,1}$, Apivich Hemachandra$^1$,\\ \textbf{Gregory Kang Ruey Lau}$^1$, \textbf{Chuan-Sheng Foo}$^2$, \textbf{Zhengyuan Liu}$^2$, \textbf{Nancy F. Chen}$^2$,\\ \textbf{Bryan Kian Hsiang Low}$^1$ \\
$^1$Department of Computer Science, National University of Singapore, Singapore\\
$^2$Institute for Infocomm Research, A*STAR, Singapore\\
\texttt{\{chenzhiliang,alfred\_leong,ongshaoyong,apivich,gregorylau\}@u.nus.edu}\\
\texttt{lowkh@comp.nus.edu.sg}
}

\begin{document}

\setcitestyle{numbers}

\maketitle
\def\thefootnote{*}\footnotetext{Equal contribution.}\def\thefootnote{\arabic{footnote}}

\begin{abstract}

Fine-tuning an LLM to maximize performance on a downstream task requires finding the optimal data and model configuration. This is a black-box optimization problem that Bayesian optimization (BO) addresses by sequentially evaluating training configurations and using observed performance feedback to adaptively guide the search towards better configurations. 
However, directly applying BO to the joint data-model space suffers from two fundamental challenges: the prohibitive cost of full LLM training at each iteration, and the large number of BO iterations required for effective exploration due to the high problem dimensionality. This paper introduces \JointSLBO{}, a BO-based approach that co-optimizes data and model configurations without requiring a full training run at every iteration. \JointSLBO{} allocates an initial portion of the optimization budget to learn a scaling-law-inspired performance predictor that estimates fully trained LLM performance from only a small number of training steps. The remaining budget is then used to run BO with this predictor, eliminating the need for full training runs and enabling \JointSLBO{} to explore significantly more data and model configurations within the same budget. We analyze \JointSLBO{}' average regret and derive the optimal budget allocation between predictor learning and BO iterations. Empirically, \JointSLBO{} outperforms independent data and model optimization methods and existing multi-fidelity BO baselines across a diverse set of downstream tasks. Our code is available at \url{https://github.com/chenzhiliang94/BO-for-LLMs}.
\end{abstract}

\section{Introduction}

Fine-tuning an LLM for a downstream task depends on two groups decision variables: which \textit{training data }to use, and how to configure the \textit{model architecture}. From the \textit{data} perspective, performance can be improved by selecting more relevant or higher-quality training data~\citep{koh2020understandingblackboxpredictionsinfluence,xia2024lessselectinginfluentialdata,data_selection_1} or by optimizing the mixture of data domains~\citep{chen2025aioliunifiedoptimizationframework,chen2025duet,liu2025regmixdatamixtureregression,xie2023doremioptimizingdatamixtures,xie2025chameleonflexibledatamixingframework}. From the \textit{model} perspective, a range of methods have been proposed to identify the most suitable architecture or fine-tuning configuration~\citep{he2024robustifyingboostingtrainingfreeneural, raschka2020modelevaluationmodelselection,bananas,zhang2024autoloraautomaticallytuningmatrix}. 

Despite their individual successes, existing data and model optimization methods are developed in isolation, ignoring their \textit{interdependence} which can influence overall performance. For instance, the optimal data configuration (e.g., training data mixture) depends on the chosen model configuration (e.g., fine-tuning configuration), but determining the optimal model configuration also depends on the chosen data configuration. This leads to a \textit{chicken-and-egg dilemma} and optimizing either component independently leads to sub-optimal LLM performance~\citep{chen2024towardsautoai}, as demonstrated empirically in Sec.~\ref{sec:exp}.

\begin{figure}[h]
\centering
\includegraphics[width=\linewidth]{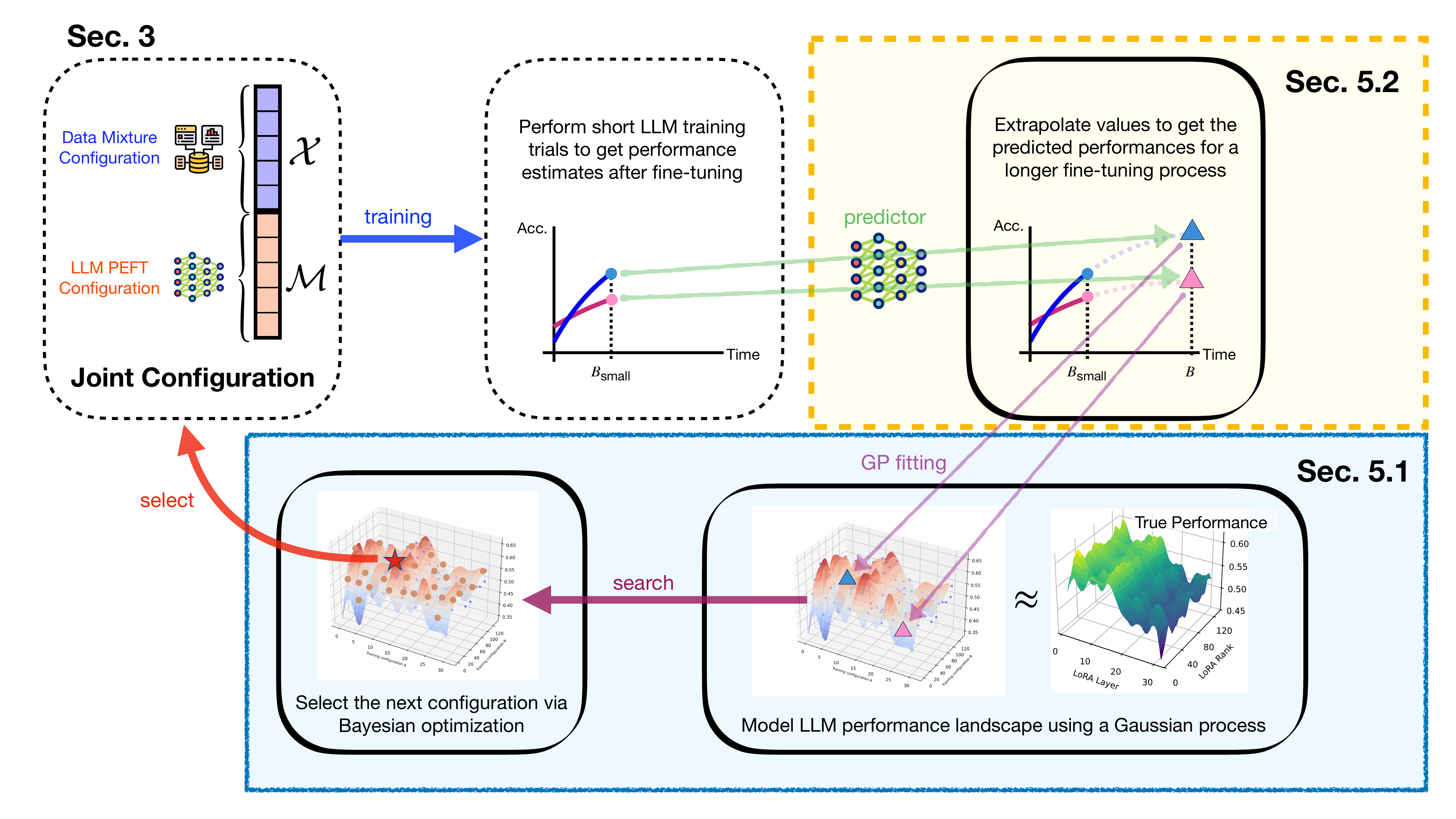}
\caption{\JointSLBO{} optimally balances budget between learning a performance predictor and running Bayesian optimization iterations.}
\label{fig:overview}
\vspace{-5mm}
\end{figure}

Our work tackles this by \textit{co-optimizing} the data and model configurations of LLMs to maximize a downstream task metric. In practice, these metrics, such as user ratings, engagement, or benchmark accuracy, are observed as feedback from the deployed LLM across successive training iterations. This formulation yields a classic \textit{black-box optimization problem} as the relationship between the training configuration and metric feedback has no closed-form expression.

Bayesian optimization (BO) is a principled approach for such problems, using observed performance feedback to adaptively guide the search toward better configurations \cite{chen2026efficientlyestimatingoptimalhyperparameter,verma-etal-2025-uncovering}.
Crucially, this iterative structure allows BO to directly exploit downstream task feedback --- an advantage that non-adaptive methods lack by design. For instance, gradient-based data selection methods \citep{xia2024lessselectinginfluentialdata} cannot utilize black-box feedback, limiting their effectiveness in this setting~\citep{chen2025duet, frazier2018tutorialbayesianoptimization, bo-gp-ucb-10}. The straightforward way of using directly apply BO over the joint data-model space suffers from two fundamental challenges: the prohibitive cost of full LLM training at every iteration, and the high dimensionality of the joint search space, which demands many iterations to explore effectively. Together, these render naive BO computationally infeasible under a finite \textit{optimization budget}.

Our method, \textbf{\underline{Jo}int \underline{B}ayesian Optimization with a \underline{S}caling-law-inspired predictor} (\JointSLBO{}), addresses both challenges by allocating an initial portion of the optimization budget to learn a scaling-law-inspired performance predictor that estimates the performance of a fully trained LLM for any training configuration from only a small number of training steps. The remaining budget is then used to run BO with this predictor, eliminating the need for full LLM training at each iteration and enabling significantly more BO iterations within the same optimization budget. Importantly, even if the predictor's estimates introduce errors, BO naturally absorbs this as observation noise, ensuring theoretical convergence guarantees (Sec.~\ref{subsec:convergence-theorem}). 

The main contributions of this paper are:

\squishlisttwo
    \item We propose \JointSLBO{}, an algorithm that co-optimizes LLM data and model configurations by running BO (Sec.~\ref{sec:BO}) with a scaling-law-inspired performance predictor (Sec.~\ref{sec:scaling-law-predictor}) that predicts the performance of a fully trained LLM for any training configuration from only a small number of training steps. This reduces the cost of each BO iteration, enabling significantly more iterations within the same optimization budget.
    \item We empirically analyze the prediction error and identify an inherent tradeoff: more initial full-training runs improve predictor accuracy, but reduce the number of subsequent BO iterations available under a fixed optimization budget (Sec.~\ref{subsec:performance-tradeof}).
    \item We theoretically analyze how predictor error propagates into \JointSLBO{}' average regret, and derive the optimal budget allocation between predictor learning and BO iterations to minimize regret (Sec.~\ref{subsec:convergence-theorem}).
    \item We empirically demonstrate that \JointSLBO{} consistently outperforms independent data and model optimization (Sec.~\ref{subsec:results-compare-data-model-opt}), vanilla BO, and multi-fidelity BO (Sec.~\ref{subsec:results-compare-mf-bo}) baselines across a diverse set of downstream tasks, with ablation studies validating our theoretical finding that an optimal budget allocation exists (Sec.~\ref{subsec:ablation}).
\squishend

\section{Related work} \label{sec:related-work}


\textbf{Data and model optimization.} Data optimization methods select or reweight training data to improve downstream performance: gradient-based selection (\textbf{LESS})~\citep{xia2024lessselectinginfluentialdata} identifies influential data points by aligning training gradients with task gradients; influence functions (\textbf{IF})~\citep{koh2020understandingblackboxpredictionsinfluence} rank data points by their estimated effect on model outputs; distributionally robust mixing (\textbf{DoReMi})~\citep{xie2023doremioptimizingdatamixtures} optimizes domain weights to perform well across task distributions; and diversity-based selection (\textbf{Diversity})~\citep{wang2024diversity-log-det} maximizes the log-determinant score of the selected subset. Model optimization methods identify suitable architectures or fine-tuning configurations: differentiable architecture search (\textbf{DARTS})~\citep{darts} relaxes discrete architecture choices into a continuous optimization; \textbf{AutoLoRA}~\citep{zhang2024autoloraautomaticallytuningmatrix} automatically optimizes LoRA rank; and \textbf{RoBoT}~\citep{he2024robustifyingboostingtrainingfreeneural} selects configurations using training-free metrics. Crucially, these methods optimize data and model configurations in isolation, ignoring their critical interdependence. Furthermore, they cannot exploit iterative black-box metric feedback (e.g., user ratings or benchmark scores) to guide optimization, and may require access to task gradients, which are unavailable in black-box settings. These methods will serve as baselines in our experiments (Sec.~\ref{sec:baseline-description}).


\textbf{Multi-fidelity BO.} BO has been widely adopted for optimizing black-box functions where evaluations are costly \citep{bo-gp-ucb-10}. Multi-fidelity BO methods \citep{lee2025costsensitivefreezethawbayesianoptimization,swersky2014freezethawbayesianoptimization,wu2019practicalmultifidelitybayesianoptimization,yen2025datamixtureoptimizationmultifidelity} incorporate fidelity as an additional input dimension and use cost-aware acquisition functions to determine which fidelity to query at each iteration. Examples of such methods include \textbf{MF-KG} \citep{wu2019practicalmultifidelitybayesianoptimization}, which uses knowledge gradients, and \textbf{MF-UCB} \citep{bo-gp-ucb-10}, which uses upper confidence bound (UCB). 
However, these methods suffer from three key limitations. First, they rely on a GP to jointly model low- and high-fidelity observations, which is often insufficiently expressive to capture the complex relationship between fidelities. Second, few works provide theoretical backing on the optimal number of queries per fidelity level. Third, freeze-thaw methods \citep{lee2025costsensitivefreezethawbayesianoptimization,swersky2014freezethawbayesianoptimization} assume monotonically improving performance via an exponential decay kernel, but this assumption breaks down when suboptimal data mixtures cause LLM performance to \textit{decrease} with more training steps (Fig.~\ref{fig:scaling-law}). Early-stopping BO \citep{bos-zhongxiang} shares the same limitation of relying on a GP across different fidelities. A detailed comparison is provided in App.~\ref{app:relation-to-mf-bo}.

\textbf{LLM scaling laws.} Existing scaling laws \citep{scaling_law_performance,hoffmann2022trainingcomputeoptimallargelanguage, kaplan2020scalinglawsneurallanguage,shukor2025scalinglawsoptimaldata, scaling_law_performance_2,zhang2024scalingmeetsllmfinetuning} establish symbolic formulas that extrapolate LLM performance from a small number of training steps. However, as shown in Sec.~\ref{subsec:scaling-predict-performance-prelim}, these formulas do not transfer to our setting as they are defined w.r.t. a fixed training configuration (e.g., a fixed training data pool), whereas we require performance predictions across a wide range of LLM training configurations. Notably, \citet{yen2025datamixtureoptimizationmultifidelity} explore the optimization of data mixtures using neural networks to predict LLM performance, but do not account for the cost of training the predictor within the optimization budget.

\section{Problem setup}
\label{sec:prelim}
We consider two types of training components: \textbf{data} $\mathcal{X}$ and \textbf{model} $\mathcal{M}$. Given these training components, we define a training process $P_B$ that fine-tunes an LLM for $B$ training steps to produce trained weights $\theta_{\mathcal{X}, \mathcal{M}, B} \triangleq P_B(\mathcal{X}, \mathcal{M})$, which are evaluated under a predefined downstream metric $\mathcal{L}$. Our formulation accommodates a broad class of downstream metrics, including black-box metrics such as benchmark accuracy or evaluation loss. Formally, this yields a black-box optimization problem over training configurations $\mathcal{X}, \mathcal{M}$:

\begin{equation}
\max_{\mathcal{X}, \mathcal{M}} \mathcal{L}(\theta_{\mathcal{X}, \mathcal{M}, B})\ . \label{eq:joint-opt}
\end{equation}

\textbf{Data $\mathcal{X}$.} We assume access to $n$ training datasets $\mathcal{D} \triangleq D_1 \cup D_2 \cup \dots \cup D_n$ drawn from $n$ distinct domains (e.g., General Knowledge, Math), and define the data component as a subset $\mathcal{X} \subseteq \mathcal{D}$. An optimal $\mathcal{X}$ selects data samples that are relevant to the target task \citep{chen2025duet} or of higher quality \citep{wang2024helpfulharmfuldatafinetuningfree,xia2024lessselectinginfluentialdata, zhang2025_diversity_data_selection}. In this work, we consider $\mathcal{X}$ as the \textit{training data mixture} \citep{chen2025duet}, represented by a probability simplex ($\mathcal{X} \in \Delta^{n-1} \subset \mathbb{R}^n$). This data mixture represents the proportion of each data domain used to train the LLM.

\textbf{Model $\mathcal{M}$.} Our formulation supports optimization over arbitrary model configurations. In this work, we focus primarily on parameter-efficient fine-tuning (PEFT) using low-rank adaptation (LoRA) \citep{hu2021lora}, but also show that \JointSLBO{} extends to full-parameter fine-tuning (see App.~\ref{app:extensions} for details). Concretely, the model configuration $\mathcal{M} \in \mathbb{R}^m$ is a vector of dimension $m$ comprising: (1) the LLM \textit{layers} to which LoRA is applied, (2) the LLM \textit{modules} to which LoRA is applied (e.g., $Q,K$-projection~\citep{attention}), and (3) the \textit{LoRA hyperparameters}, including rank, $\alpha$, and dropout \citep{hu2021lora}.

\textbf{Optimization budget $C$.} A key aspect of our work is the incorporation of an explicit total optimization budget $C \geq B$, which prior work \citep{yen2025datamixtureoptimizationmultifidelity} does not consider. Without such a budget constraint, brute-force search could be applied over all possible training configurations \citep{bergsma2025powerlinesscalinglaws,li2025predictablescaleiifarseer}. For instance, given a budget of $C=50000$ steps with each full training run costing $B=1000$ steps, only $C/B=50$ training configurations can be evaluated. Our algorithm accommodates other budget metrics as well, such as FLOPs or tokens, as detailed in App.~\ref{app:practical-consideration}.

\section{Preliminary findings}
\label{sec:prelim-findings}

To motivate the design of \JointSLBO{}, we first present empirical findings on the structure of the performance landscape and the feasibility of using neural networks as performance predictors.

\subsection{Performance landscape}

\begin{wrapfigure}{R}{0.5\textwidth}
  \vspace{-5mm}
  \centering
  \begin{subfigure}[t]{0.48\linewidth}
    \centering
    \includegraphics[width=\linewidth]{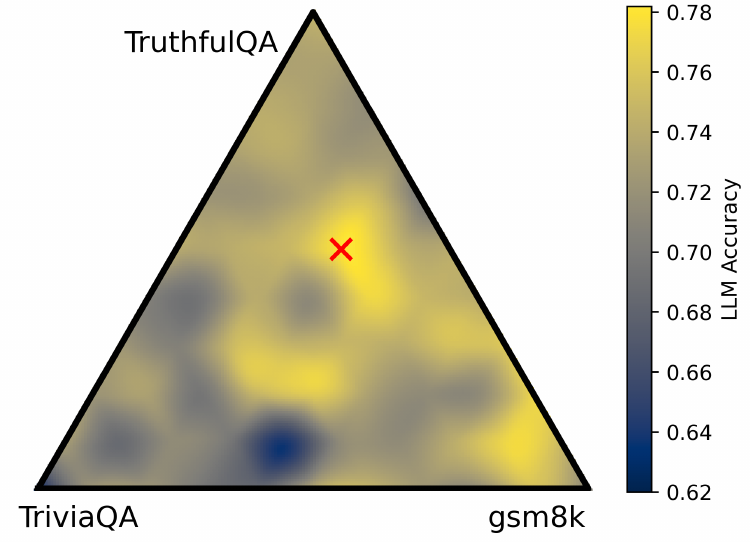}
    \caption{LLM performance w.r.t. data mixtures.}
    \label{fig:data-mix}
  \end{subfigure}
  \hfill
  \begin{subfigure}[t]{0.48\linewidth}
    \centering
    \includegraphics[width=\linewidth]{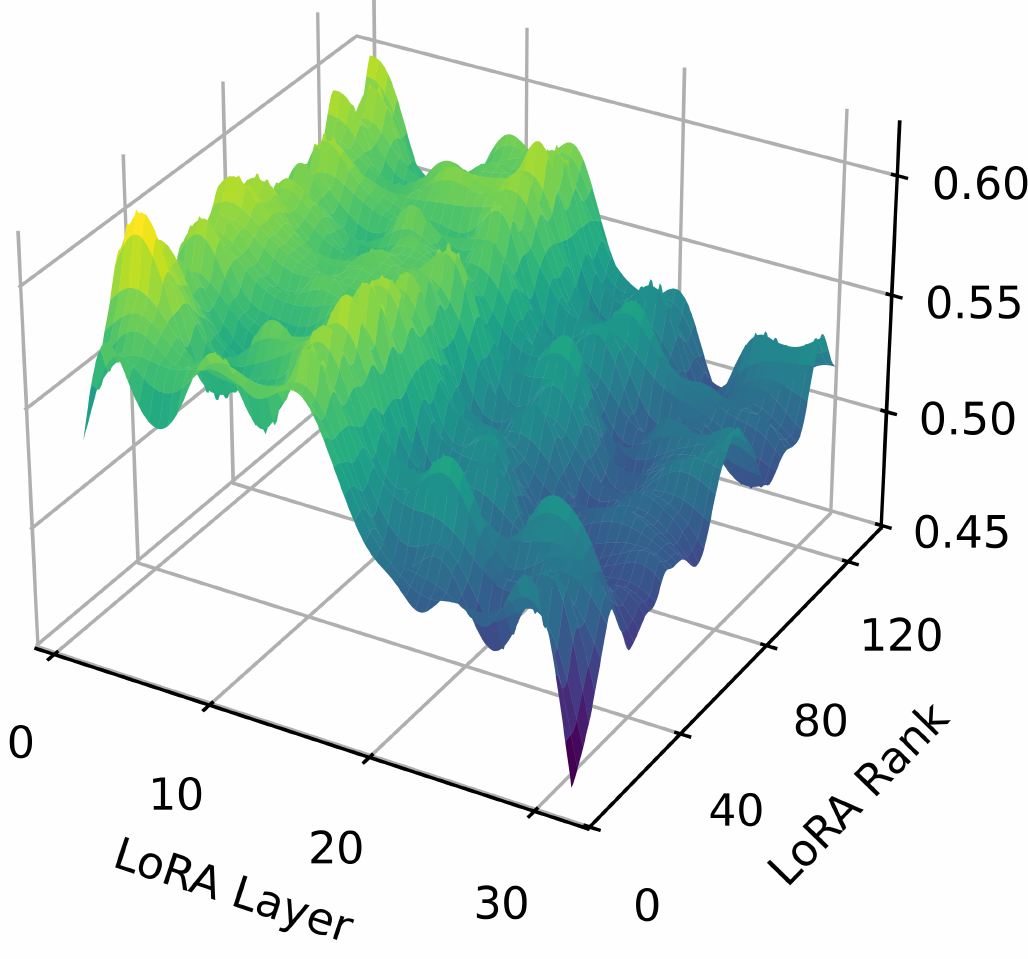}
    \caption{LLM performance w.r.t. LoRA configurations.}
    \label{fig:lora}
  \end{subfigure}
  \caption{LLM performance varies with different configurations.}
  \label{fig:prelim:model}
\end{wrapfigure}

We examine the performance landscape of our optimization problem. In Fig.~\ref{fig:prelim:model}, we fine-tune a \texttt{Llama-3-8B-Instruct}~\citep{grattafiori2024llama} model across varying data mixtures and LoRA configurations and use GSM$8$K~\citep{gsm8k} accuracy as the downstream metric (averaged over 5 trials). The landscape reveals two key properties. First, jointly optimizing LoRA layer and rank yields over $10\%$ performance improvement over an arbitrarily chosen configuration (Fig.~\ref{fig:lora}). Second, the landscape varies rather smoothly across the input dimensions (Fig.~\ref{fig:data-mix}), suggesting that it can be effectively modeled by a smooth surrogate function \citep{frazier2018tutorialbayesianoptimization}, making BO well-suited for this problem.

\subsection{Extrapolating LLM performance}
\label{subsec:scaling-predict-performance-prelim}

Next, we investigate whether LLM performance can be predicted from only a small number of training steps. Fig.~\ref{fig:scaling-law} shows that performance trajectories vary substantially across different training configurations.

\begin{wrapfigure}{R}{0.45\textwidth}
\centering
\vspace{-2mm}
\includegraphics[width=\linewidth]{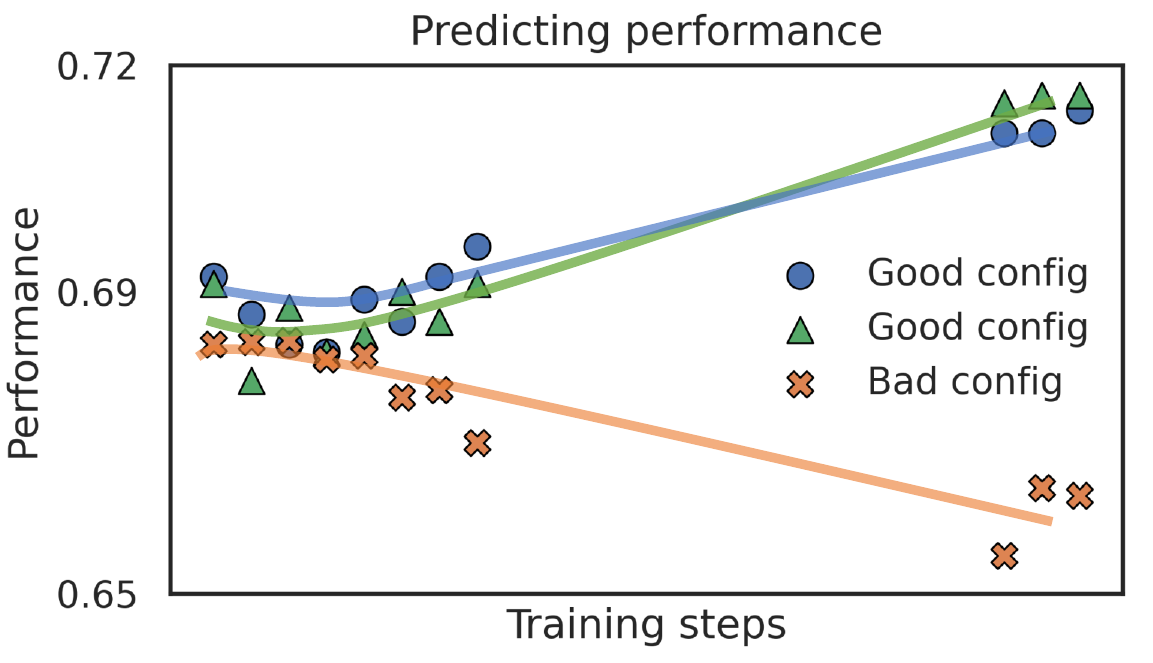}
\caption{Training a neural network to predict fully trained LLM performance from a small number of training steps.}
\label{fig:scaling-law}
\end{wrapfigure}

For instance, a good training configuration (\textcolor{ForestGreen}{green}) yields improved performance as training steps increase, whereas a suboptimal training configuration (\textcolor{orange}{orange}) causes performance to \textit{decrease}. This is unsurprising as poorly chosen data mixtures can degrade LLM performance on downstream tasks. To capture this variability, \JointSLBO{} uses a flexible neural network (Sec.~\ref{sec:scaling-law-predictor}) that predicts performance from only a small number of training steps, across diverse training configurations. By contrast, existing scaling law formulas \citep{scaling_law_performance,zhang2024scalingmeetsllmfinetuning} are defined for \textit{fixed} training configurations and cannot be applied here. While our predictor lacks the symbolic representation of classical scaling laws, its primary role is to provide useful signals to guide optimization.

\section{Introducing \JointSLBO{}}\label{sec:approach}

\JointSLBO{} consists of two main steps: (i) \textbf{Predictor learning.} We allocate an initial portion of the budget to learn a performance predictor that estimates the fully trained LLM performance from a small number of training steps (Sec.~\ref{sec:scaling-law-predictor}). (ii) \textbf{BO with predictor.} We iteratively evaluate candidate training configurations following the standard BO feedback loop (Sec.~\ref{sec:BO}). Unlike vanilla BO, which requires the LLM to be trained to completion at each iteration, \JointSLBO{} queries the predictor to estimate the fully trained LLM performance after only a small number of training steps.

A key question is \textit{how much of the optimization budget to allocate to learning the predictor}. This involves a clear tradeoff: allocating more budget to learn the predictor leaves fewer BO iterations to effectively search for the optimal training configuration, while reducing the learning budget results in a less accurate predictor that does not faithfully recover the true performance landscape. We analyze the performance tradeoff between the number of initial full training runs and remaining BO iterations both theoretically (Sec.~\ref{subsec:convergence-theorem}) and empirically (Sec.~\ref{subsec:ablation}), deriving the optimal budget allocation. Notably, \JointSLBO{} does not require a perfect performance predictor because the BO framework naturally treats the prediction error as \textit{observation noise}, and we show theoretically that \JointSLBO{} converges to the optimal configuration despite this noise.

\subsection{BO as the backbone of \JointSLBO{}}
\label{sec:BO}

We first provide an overview of the standard BO algorithm in a sequential feedback setting. We consider LLM performance as a function $\mathcal{L} : \mathbb{R}^d \to \mathbb{R}$ over the space of inputs $x = [\mathcal{X}, \mathcal{M}] \in \mathbb{R}^d$ where $d = n+m$ (see Sec.~\ref{sec:prelim}). We treat our objective function in Eq.~\ref{eq:joint-opt} as a \textit{black-box function} whose maximizer $x^* \triangleq \argmax_{x} \mathcal{L}(x)$ we want to recover.
In line with existing work, we model $\mathcal{L}$ as a surrogate \emph{Gaussian process} (GP) \citep{gp-for-ml}.
At each iteration $t=1,2,\dots,T$, we use an LLM training configuration $x_t$ to obtain a \textit{noisy} realization (after training with $x_t$) of the LLM performance $y_t \triangleq \mathcal{L}(x_t) + \epsilon_t$, which we assume is corrupted with sub-Gaussian noise $\epsilon_t$ (e.g., normally-distributed noise) to form the sample $(x_t,y_t)$. Consistent with the work of~\citet{bo-kernelized-bandits}, our
GP is specified by its \emph{prior} mean $\mu(x)$ and covariance $\kappa(x,x')$ for all $x,x' \in \mathbb{R}^d$, where $\kappa$ is a \textit{kernel} function that characterizes the correlation between two inputs $x$ and $x'$.

Given the noisy observations $\bm{y}_t \triangleq [y_{\tau}]^{\top}_{\tau=1,\dots,t}$ at inputs $x_1,\dots,x_t$, the posterior belief of $\mathcal{L}$ at any new input $x'$ is a Gaussian distribution with the \emph{posterior} mean and variance given by
\begin{equation}
\label{gp:posterior}
\begin{split}
     & \mu_t(x') \triangleq \kappa_t^{\top}(x')(K_t + \zeta I)^{-1}\bm{y}_t \\
     & \sigma_t^2(x') \triangleq \kappa(x',x')-\kappa_t^{\top}(x')(K_t + \zeta I)^{-1}\kappa_t(x')
\end{split}
\end{equation}
where $\kappa_t(x') \triangleq [\kappa(x', x_{\tau})]^{\top}_{\tau=1,\dots,t}$ is a column vector, $K_t \triangleq [\kappa(x_\tau,x_{\tau'})]_{\tau,\tau' \in 1,\ldots,t}$ is a $t \times t$ covariance matrix, and $\zeta > 0$ is a free hyperparameter \citep{bo-kernelized-bandits}. By learning the correlation between inputs and observations, modeling $\mathcal{L}$ with a GP allows us to capture the relationship between training configurations and LLM performance.

\textbf{Using BO for data and model co-optimization.}
At each iteration, BO proposes a candidate data and model configuration by maximizing an acquisition function (e.g., \emph{upper confidence bound} (UCB) \citep{bo-gp-ucb-10}); the LLM is then trained under that configuration, and its downstream task performance is observed.
We measure convergence via average regret after $T$ iterations, given by $\mathcal{R}_T \triangleq T^{-1}\sum_{t=1}^T (\mathcal{L}(x^*)-\mathcal{L}(x_t))$, where $\mathcal{L}(x^*)$ is the optimum~\citep{tay2023bayesian_cost}. More details on BO are provided in App.~\ref{app:BO}.

\subsection{Amortization with performance predictor}\label{sec:scaling-law-predictor}

Directly applying BO requires training an LLM to completion at every iteration, which is computationally prohibitive under a finite optimization budget. To amortize this cost, we learn a performance predictor that estimates fully trained LLM performance from a small number of training steps, allowing us to run significantly more BO iterations within the same budget. 
In Sec.~\ref{subsec:convergence-theorem}, we further provide a theoretical analysis of the computational cost tradeoff associated with learning this predictor, which is absent in prior literature \citep{amortize_BO_1, yen2025datamixtureoptimizationmultifidelity}.

For the predictor to generalize, it must estimate LLM performance for different training configurations, which are not known in advance. To this end, \JointSLBO{} learns a neural network that takes \textit{any} training configuration $[\mathcal{X}, \mathcal{M}]$ and a sequence of performance observations $\{{\mathcal{L}}(\theta_{\mathcal{X}, \mathcal{M}, b})\}_{b\leq\Bsmall}$ up to $\Bsmall < B$ training steps, and predicts the fully trained LLM performance $\mathcal{L}(\theta_{\mathcal{X}, \mathcal{M}, B})$. For instance, in our experiments we use multiple sequential observations up to $\Bsmall=100$ steps to predict final performance at $B=1000$ steps. Examples of predictions are shown in Fig.~\ref{fig:scaling-law}. As expected, our ablation studies in Sec.~\ref{subsec:ablation} show that a larger $\Bsmall$ yields more accurate predictions at the expense of consuming more optimization budget.

Learning the predictor consists of two stages.
(i) \textbf{Data collection:} To ensure diversity, we sample a random Sobol sequence~\citep{sobol} of $N$ training configurations over $\mathcal{X}$ and $\mathcal{M}$. For each configuration, we train the LLM to completion, recording performance at several sequential steps up to $\Bsmall$ as well as the final performance ${\mathcal{L}}(\theta_{\mathcal{X}, \mathcal{M}, B})$ at $B$ steps; these full training observations also warm-start the GP surrogate (Sec.~\ref{sec:BO}). 
(ii) \textbf{Training:} \JointSLBO{} fits a neural network
$\mathcal{F} : (\mathcal{X}, \mathcal{M}, \{{\mathcal{L}}(\theta_{\mathcal{X}, \mathcal{M}, b})\}_{b\leq\Bsmall}) \mapsto {\mathcal{L}}(\theta_{\mathcal{X}, \mathcal{M}, B})$ on these observations,
mapping performance after a small number of training steps to fully trained LLM performance.
Once learned, $\mathcal{F}$ replaces full training runs: at every subsequent BO iteration, the LLM is trained for only $\Bsmall$ steps and $\mathcal{F}$ predicts the fully trained LLM performance.

\subsection{Prediction error and performance tradeoff}\label{subsec:performance-tradeof}

\begin{wrapfigure}{R}{0.4\textwidth}
\centering
\vspace{-5mm}
\includegraphics[width=\linewidth]{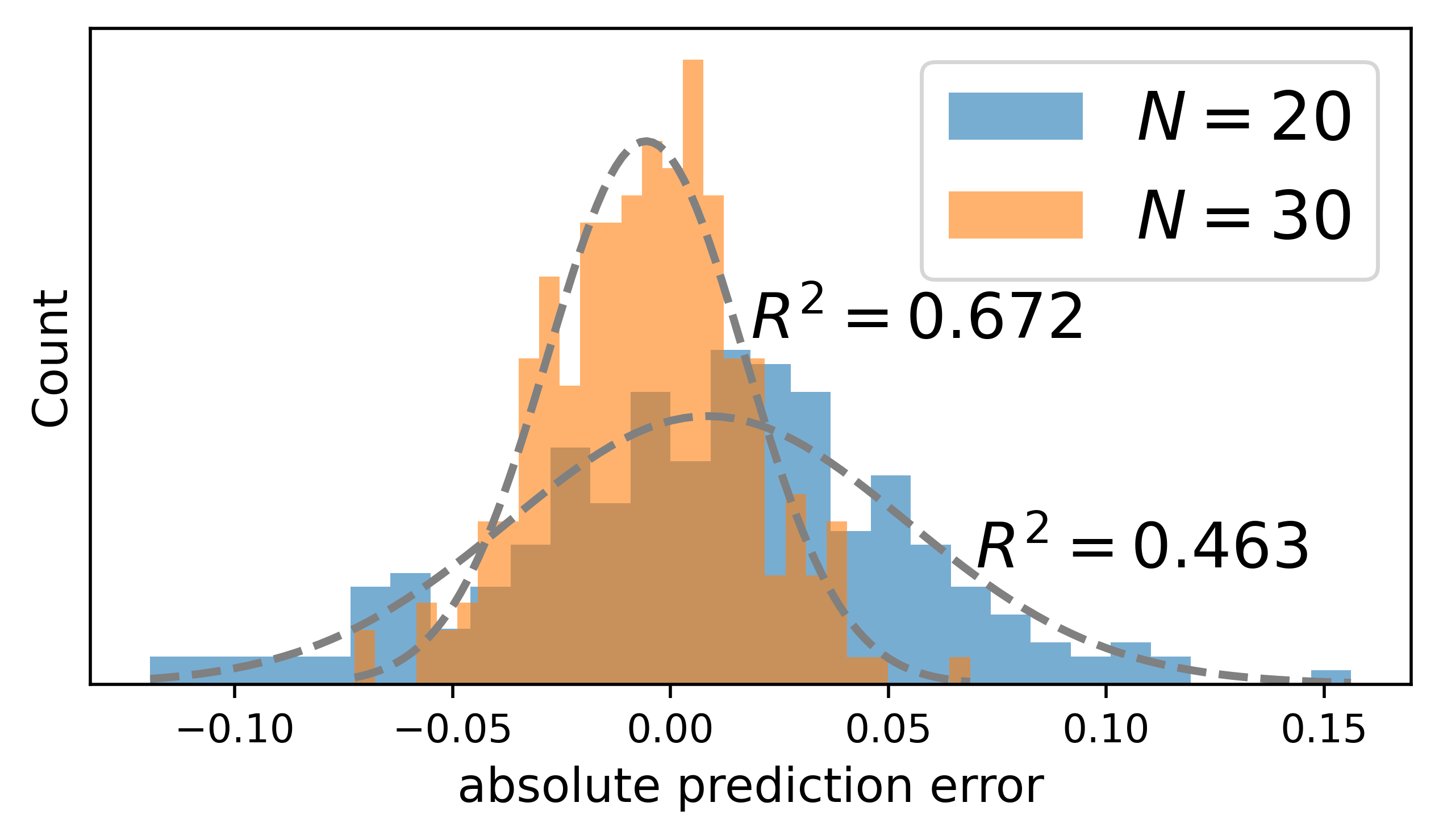}
\caption{Predictor error w.r.t. $N$}
\vspace{-3mm}
\label{fig:prediction-error-f-samples-N}
\end{wrapfigure}

\textbf{Prediction error.}
The accuracy of $\mathcal{F}$ directly influences the quality of BO iterations in \JointSLBO{}, making it important to characterize how prediction error behaves in practice. Fig.~\ref{fig:prediction-error-f-samples-N} shows that as $N$ increases, prediction accuracy improves, reflected in higher $R^2$ scores and tighter error concentration around zero. Empirically, the prediction error is well-approximated by a normal distribution, whose variance decreases with increasing $N$. Crucially, \JointSLBO{} converges both theoretically (Theorem~\ref{thm:convergence-main}) and empirically (Fig.~\ref{fig:BO-convergence}), despite this prediction noise.

\textbf{Tradeoff.} 
However, increasing $N$ comes at a cost. Under a fixed optimization budget $C$, collecting $N$ full training runs each requiring $B$ steps consumes an initial budget of $N\times B$, leaving only $\lfloor (C - NB) / \Bsmall \rfloor$ BO iterations remaining. Thus, a larger $N$ produces a more accurate predictor but allows fewer BO iterations, and vice versa. In the next section, we formally analyze how the choice of $N$ influences the convergence of \JointSLBO{}.

\subsection{Theoretical analysis of \JointSLBO{}}\label{subsec:convergence-theorem}

We analyze how the regret bound of \JointSLBO{} varies w.r.t. the prediction error and the choice of $N$. While the prediction error cannot be precisely quantified analytically, Fig.~\ref{fig:prediction-error-f-samples-N} suggests it is $R$-sub-Gaussian \citep{bo-kernelized-bandits} with increasing variance as $N$ decreases. The following assumption formalizes this characteristic based on the empirical observations in Fig.~\ref{fig:prediction-error-f-samples-N}.

\textbf{Assumption.} Let $\epsilon$ be the prediction error of the performance predictor learned from $N$ full training runs, whose empirical distribution is shown in Fig.~\ref{fig:prediction-error-f-samples-N}. We assume $\epsilon$ is $R$-sub-Gaussian with $R=\frac{k}{\sqrt{N}}$ for some constant $k$, reflecting that prediction error decreases as $N$ grows.

With this assumption, the following theorem characterizes \JointSLBO{}' average regret given an optimization budget $C$ and $N$ initial full training runs.

\begin{restatable}{theorem}{convergence}
\label{thm:convergence-main}
Let ${\mathcal{L}}(\theta_{\mathcal{X}, \mathcal{M}, B})$ be the performance landscape of the LLM training configuration with bounded RKHS norm: $\left \lVert \mathcal{L} \right \rVert _{\kappa} = \sqrt{	\langle \mathcal{L},\mathcal{L} \rangle_\kappa} \leq \mathcal{B}$
w.r.t. kernel $\kappa$. Assume our performance predictor is learned from $N$ full training runs and makes predictions $\mathcal{F}(\mathcal{X}, \mathcal{M}, \{{\mathcal{L}}(\theta_{\mathcal{X}, \mathcal{M}, b})\}_{b\leq\Bsmall}) = {\mathcal{L}}(\theta_{\mathcal{X}, \mathcal{M}, B}) + \epsilon$ from $\Bsmall < B$ training steps where $\epsilon$ satisfies the assumption above. Then, running \JointSLBO{} over LLM training configurations $\mathcal{X}, \mathcal{M}$ until an optimization budget of $C$ is exhausted yields $T = \floor*{\frac{C-NB}{\Bsmall}}$ BO iterations, and with the IGP-UCB acquisition function \citep{bo-kernelized-bandits}, \JointSLBO{} achieves the following average regret with probability at least $1-\delta$:
\begin{equation}
\label{eq:avg-regret}
\begin{aligned}
\mathcal{R}_T
&= \mathcal{O}\Bigg(
\sqrt{\frac{\Bsmall}{C - NB - \Bsmall}}
\Big(
\mathcal{B}\sqrt{\gamma_T}
+ \frac{k}{\sqrt{N}}
\sqrt{\gamma_T^2 + \gamma_T \ln(1/\delta)}
\Big)
\Bigg),
\end{aligned}
\end{equation}
where $\gamma_{T}$ is the maximum information gain of $\mathcal{L}$ after $T = \floor*{\frac{C-NB}{\Bsmall}}$ BO iterations.
\end{restatable}

Theorem~\ref{thm:convergence-main} implies \JointSLBO{} is essentially no-regret, as the average regret reduces to zero with increasing budget $C$. The proof is provided in App.~\ref{app:proof-BO-convergence}. This bound is also tighter than the average regret bound of vanilla BO \citep{bo-kernelized-bandits} under some reasonable assumptions (see App.~\ref{app:comparison-vanilla-BO}).

\textbf{Optimal budget allocation}. Theorem~\ref{thm:convergence-main} presents a compute-performance tradeoff: a larger $N$ yields more accurate predictions, \textit{reducing} the $k/\sqrt{N}$ term in Eq.~\ref{eq:avg-regret}, but also reduces the number of BO iterations available, \textit{increasing} the left term in Eq.~\ref{eq:avg-regret}. Using the experimental values $\Bsmall=100$, $C=50000$, and $B=1000$, \textbf{the optimal $N$} \textbf{that minimizes the average regret $\mathcal{R}_T$ lies approximately between $25$ and $35$} (depending on other constants). We empirically illustrate this tradeoff and provide practical guidelines for choosing $N$ in our ablation studies (Sec.~\ref{subsec:ablation}).


\section{Experiments}\label{sec:exp}
We empirically validate \JointSLBO{} across a variety of downstream task metrics and settings. Our experiments are organized into three parts: \textbf{(i)} comparison against all combinations of independent data and model optimization methods, \textbf{(ii)} comparison against vanilla and multi-fidelity BO baselines, and \textbf{(iii)} ablation studies to isolate the contributions of key design choices, and computational cost analysis.

In our main experiments, we fine-tune \texttt{Llama-3-8B-Instruct} with LoRA \citep{hu2021lora}. Extensions to different model families (\texttt{Qwen}), larger models (14B, 32B), and full-parameter fine-tuning are provided in App.~\ref{app:extensions}, where \JointSLBO{} remains consistently effective. Our objective is to maximize downstream task performance after $B=1000$ training steps, and within an optimization budget of  $C=50000$ training steps. We evaluate benchmark accuracy using \texttt{lm-evaluation-harness}~\citep{eval-harness} across six language tasks (Fig.~\ref{fig:BO-convergence}). We use $N=30$ initial full training runs to train the performance predictor, which makes predictions after $\Bsmall=100$ steps --- enabling $4\times$ more BO iterations within the same budget. We use the UCB acquisition function~\citep{bo-gp-ucb-10} with an increasing $\beta_t$ schedule, and a mixed kernel that combines a Hamming-distance kernel with an SE kernel (App.~\ref{app:practical-consideration}). Ablation studies in Sec.~\ref{subsec:ablation} justify the choices of $N$ and $\Bsmall$, and full experimental details are in App.~\ref{app:exp-details}.

\textbf{Data configuration.} Data configuration $\mathcal{X}$ is a mixture over 9 training domains: \textbf{Wikitext} \citep{wikitext-data}, \textbf{GSM$8$K} \citep{gsm8k}, \textbf{PubmedQA} \citep{pubmedqa}, \textbf{SciQ} \citep{sciq}, \textbf{TriviaQA} \citep{triviaQA}, \textbf{TruthfulQA} \citep{truthfulQA}, \textbf{MMLU} \citep{hendrycks2021measuringmassivemultitasklanguage},
\textbf{AI2 ARC} \citep{ai2_arc}, and \textbf{CommonsenseQA} \citep{talmor2019commonsenseqaquestionansweringchallenge}. We construct a fine-tuning dataset consisting of $10000$ data points by randomly sampling from these training datasets \citep{chen2025duet,xie2023doremioptimizingdatamixtures,ye2024datamixinglawsoptimizing}, with train and test splits kept separate to prevent data contamination. 

\textbf{Model configuration.} Model configuration $\mathcal{M}$ comprises five LoRA hyperparameters: the LLM layers to which LoRA is applied, the LLM modules to which LoRA is applied (e.g., $Q,K$-projection), rank, $\alpha$, and dropout. This gives a total of $10$ model configuration dimensions. Further experimental details are provided in App.~\ref{app:exp-details}.

\subsection{Baselines}\label{sec:baseline-description}

We compare \JointSLBO{} against three categories of baselines (described in Sec.~\ref{sec:related-work}). \textbf{Data optimization:} LESS, DoReMi, Influence Function (IF), and Diversity. \textbf{Model optimization:} DARTS, AutoLoRA, and RoBoT. \textbf{BO methods:} vanilla BO, MF-KG, and MF-UCB, with discrete fidelities of $\Bsmall = 100$ and $B=1000$. Since task gradients are not available in our black-box setting, gradient-dependent baselines use training gradients as a proxy. Further implementation details are provided in App.~\ref{app:baseline-details}.

Since \JointSLBO{} is, to our knowledge, the first method to co-optimize data and model configurations with iterative black-box feedback, no existing baseline offers a direct comparison in this setting. Our experiments therefore serve two purposes: \textbf{(i)} to assess the value of co-optimization by comparing against all combinations of data and model optimization baselines, and \textbf{(ii)} to evaluate the advantage of our predictor-based approach over kernel-based multi-fidelity BO methods.

\subsection{Comparison with data and model optimization methods}
\label{subsec:results-compare-data-model-opt}

\begin{table}[h]
  \centering
  \vspace{-5mm}
  \caption{GSM$8$K accuracy (higher is better) for all combinations of data and model optimization methods vs. \JointSLBO{}.}
  \vspace{1mm}
  \label{table:table-mix-and-match-gsm8k}
  \resizebox{\columnwidth}{!}{%
    \begin{tabular}{@{}llccccccc@{}}
      \toprule
      \multicolumn{2}{c}{\bfseries $\downarrow$ Model | Data $\rightarrow$}
        & Default & LESS & DoReMi & IF & Diversity & \JointSLBO{} \\
      \midrule
        & Default  & $71.9_{\pm 1.5}$ & $71.1_{\pm1.2}$ & $72.3_{\pm3.2}$ & $68.7_{\pm 1.0}$ & $74.4_{\pm 2.0}$ & - \\
        & DARTS        & $73.1_{\pm0.9}$ & $71.7_{\pm0.7}$ & $74.7_{\pm1.4}$ & $69.3_{\pm0.5}$ & $66.8_{\pm0.8}$ & - \\
        & AutoLoRA     & $72.9_{\pm1.2}$ & $75.2_{\pm0.4}$ & $70.9_{\pm0.8}$ & $68.5_{\pm0.5}$ & $74.0_{\pm0.6}$ & - \\
        & RoBoT        & $71.8_{\pm0.7}$ & $72.7_{\pm1.6}$ & $74.0_{\pm1.9}$ & $73.1_{\pm1.6}$ & $70.2_{\pm1.8}$ & - \\
        & \JointSLBO{} & - & - & - & - & - &
          \cellcolor{darkgreen!50}\textbf{\underline{86.0}}$_{\pm 0.7}$ \\
      \bottomrule
    \end{tabular}%
  }
\end{table}

We evaluate all combinations of conventional data and model optimization methods applied independently to \texttt{Llama-3-8B-Instruct}. These are strong baselines under the standard paradigm of independent data and model optimization. However, they \textit{cannot} exploit iterative black-box feedback or model the interdependence between data and model configurations, which are precisely the capabilities that \JointSLBO{} introduces. Table~\ref{table:table-mix-and-match-gsm8k} shows that \JointSLBO{} consistently outperforms all such combinations, achieving the best accuracy of $86.4\%$ on GSM$8$K and highlighting the clear value of joint co-optimization with iterative feedback. Additional results for other language tasks are provided in App.~\ref{app:more-data-model-exps}.

\subsection{Comparison with BO methods}\label{subsec:results-compare-mf-bo}

\begin{figure}[ht]
\centering
\includegraphics[width=\linewidth]{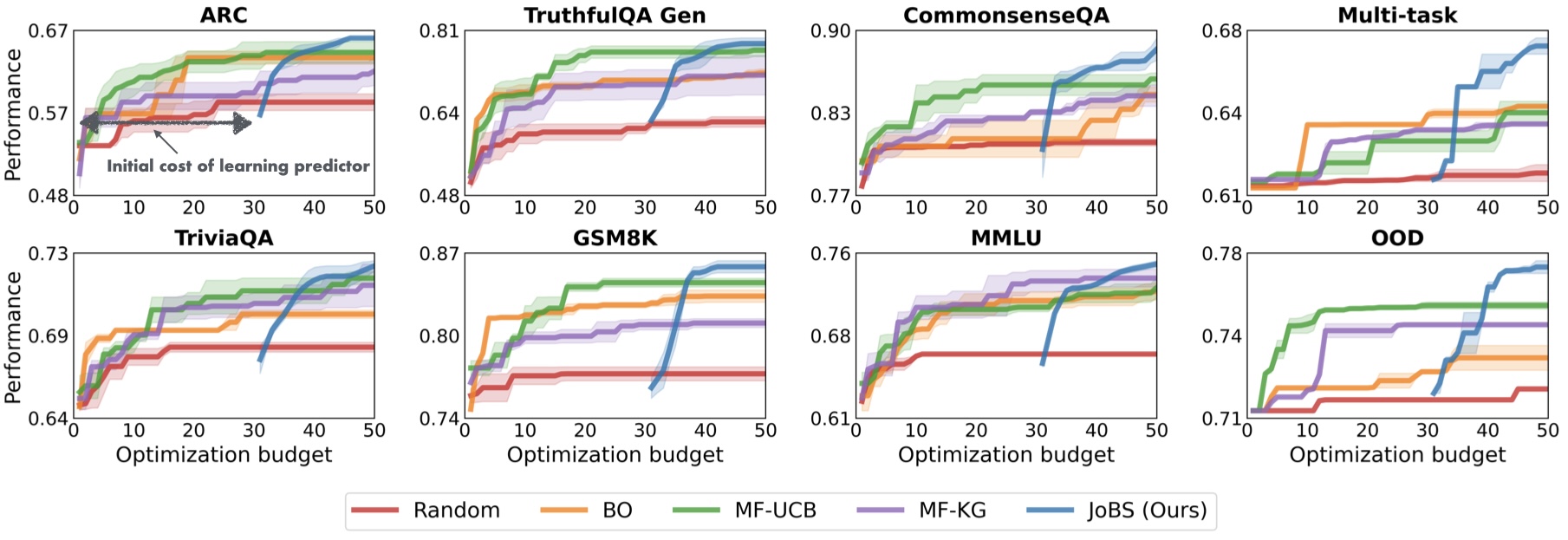}
\caption{Best-found LLM performance against optimization budget (training steps, in thousands) for \JointSLBO{} and BO baselines across six language tasks, as well as out-of-domain (OOD) and multi-task settings. \JointSLBO{} consistently achieves the highest performance across all settings. \JointSLBO{}' curve begins later as an initial portion of the budget is allocated to predictor learning.}
\label{fig:BO-convergence}
\end{figure}

\textbf{In-domain setting.} We compare \JointSLBO{} against the three BO baseline methods (\textbf{vanilla, MF-KG, MF-UCB}) described in Sec.~\ref{sec:baseline-description} across six language tasks, where the downstream evaluation task is part of the training domain. All methods are allocated the same optimization budget of $C=50000$ steps, with two discrete fidelities at $\Bsmall=100$ and $B=1000$ steps for multi-fidelity methods. Note that \JointSLBO{}' curve begins later as an initial portion of the budget is allocated to predictor learning.
As shown in Fig.~\ref{fig:BO-convergence}, \JointSLBO{} consistently outperforms all baselines across all six tasks. Outperforming vanilla BO confirms that reducing the per-iteration cost via the performance predictor enables more effective exploration within the same budget, while outperforming multi-fidelity BO baselines shows that the performance predictor is more effective than using a GP kernel for modeling the relationship between low- and high-fidelity observations~\citep{wu2019practicalmultifidelitybayesianoptimization}. Further details on the comparison with multi-fidelity BO are provided in App.~\ref{app:relation-to-mf-bo}.

\textbf{Out-of-domain setting.}
We evaluate \JointSLBO{} in an out-of-domain (OOD) setting where the target benchmark (GSM$8$K) is withheld from the pool of training data domains, forcing the optimizer to rely on knowledge transfer from related domains. 
As shown in Fig.~\ref{fig:BO-convergence} (bottom-right), \JointSLBO{} outperforms all BO baselines even in this harder setting, consistently finding configurations that generalize to the held-out domain. This indicates that the performance predictor generalizes beyond the training domains, and \JointSLBO{} remains effective provided the training domains have sufficient coverage of the target task distribution, which is a reasonable assumption in practical deployment scenarios.

\textbf{Multi-task setting.}
We evaluate \JointSLBO{} in a multi-task setting where the downstream metric is an aggregation of multiple benchmarks to demonstrate its flexibility to arbitrary black-box objectives. Specifically, we evaluate on five benchmarks---GSM$8$K, TruthfulQA, TriviaQA, CommonsenseQA, and MMLU---using the test splits that are excluded from training, and aggregate their accuracies into a single composite objective (average normalized score).
As shown in Fig.~\ref{fig:BO-convergence} (top-right), \JointSLBO{} remains effective in this setting, highlighting that \JointSLBO{} is agnostic to the form of the downstream metric and naturally extends to arbitrary black-box objectives, including composite ones. \JointSLBO{} is also compatible with the multi-objective BO setting~\citep{daulton2022multiobjectivebayesianoptimizationhighdimensional}, though a thorough investigation is left for future work.

\textbf{Extended settings}. We further evaluate \JointSLBO{} on different model families (\texttt{Qwen}), larger models (14B, 32B), full-parameter fine-tuning, and evaluation loss as the metric. As shown in App.~\ref{app:extensions}, \JointSLBO{} remains consistently effective across different model families, model sizes, training settings, and downstream metrics, demonstrating its flexibility to various settings.

\subsection{Ablation studies and computational cost}
\label{subsec:ablation}

We conduct several ablation studies (using \texttt{Llama-3-8B-Instruct} evaluated on CommonsenseQA) to investigate key design choices in \JointSLBO{}, and additionally analyze its computational cost.

\begin{figure}[ht]
  \centering
  \vspace{-2mm}
  \begin{subfigure}[t]{0.30\textwidth}
    \includegraphics[width=\linewidth]{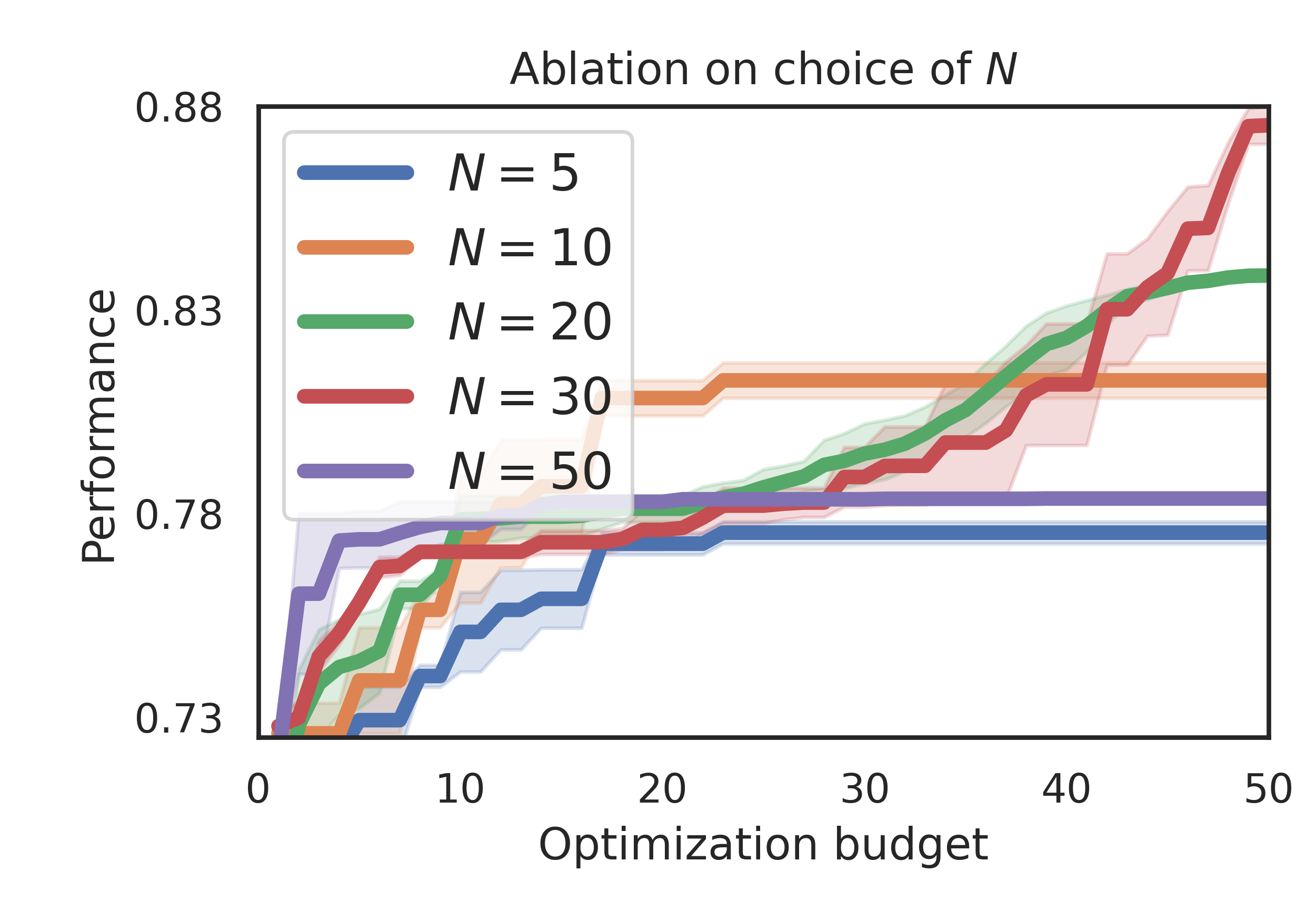}
    \caption{Varying $N$}
    \label{fig:ablation-varying-N}
  \end{subfigure}
  \hfill
  \begin{subfigure}[t]{0.30\textwidth}
    \includegraphics[width=\linewidth]{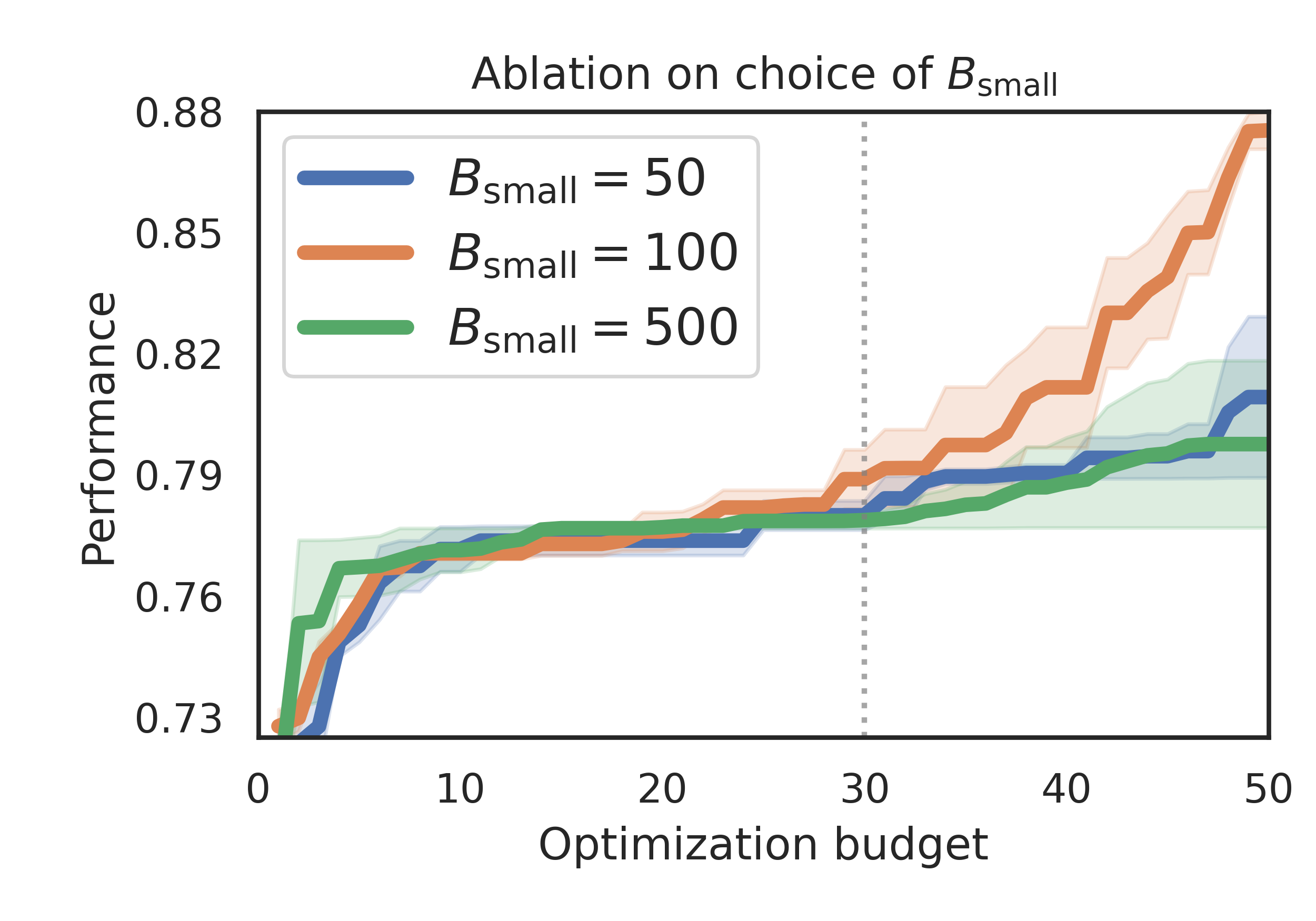}
    \caption{Varying $\Bsmall$}
    \label{fig:ablation-varying-b_small}
  \end{subfigure}
  \hfill
  \begin{subfigure}[t]{0.30\textwidth}
  \centering
  \vspace{-27mm}
  \footnotesize
  \setlength{\tabcolsep}{8pt}
  \begin{tabular}{lc}
    \toprule
    \bfseries Method & \bfseries Time (hours) \\
    \midrule
    LESS        & $16.3$ \\
    DoReMi      & $18.5$ \\
    IF          & $52$ \\
    Diversity   & $122$ \\
    \JointSLBO{} & \cellcolor{darkgreen!50}13.3 \\
    \bottomrule
  \end{tabular}
  \vspace{3mm}
  \caption{Wall-clock runtime}
  \label{fig:wall-clock}
  \end{subfigure}
  \caption{(a) Ablation on choice of $N$. (b) Ablation on choice of $\Bsmall$. (c) Wall-clock runtime of \JointSLBO{} vs. data optimization methods.}
\end{figure}

\textbf{Effects of varying $N$.} Fig.~\ref{fig:ablation-varying-N} illustrates the tradeoff between the number of initial full training runs $N$ used to train the performance predictor $\mathcal{F}$ and the number of available BO iterations under a fixed budget. A larger $N$ improves predictor accuracy, but consumes more budget, while a smaller $N$ allows more BO iterations but yields a less accurate predictor. As such, a value of $N$ that is too small or too large degrades \JointSLBO{}' performance. Empirically, $N=30$ strikes the best balance under our experimental setting, which is consistent with our theoretical findings in Sec.~\ref{subsec:convergence-theorem}.

\textbf{Effects of varying $\Bsmall$.}  Fig.~\ref{fig:ablation-varying-b_small} shows a similar tradeoff for $\Bsmall$. A larger $\Bsmall$ gives more accurate predictions but reduces the number of available BO iterations, while a smaller $\Bsmall$ allows more BO iterations but yields a less accurate predictor. From our experiments, we find $\Bsmall=100$ steps (10\% of $B=1000$) provides the best balance between prediction accuracy and BO iterations.

\textbf{Computational costs.} Lastly, Fig.~\ref{fig:wall-clock} shows that \JointSLBO{} has a lower wall-clock runtime than all data optimization baselines under the same budget $C=50000$ steps. A detailed computational cost analysis is provided in App.~\ref{app:compute-cost}, and qualitative case studies of the optimal training configurations found by \JointSLBO{} are presented in App.~\ref{app:qualitative-config-found-analysis}.

\section{Conclusion and future work}
We introduced \JointSLBO{}, a BO-based approach that co-optimizes data and model configurations for LLMs. \JointSLBO{} learns a scaling-law-inspired performance predictor to estimate fully trained LLM performance from only a small number of training steps, eliminating the need for full LLM training at every iteration. We showed that \JointSLBO{} enjoys theoretical convergence guarantees and derived the optimal budget allocation between predictor learning and BO iterations. One limitation is that \JointSLBO{} requires iterative deployment to optimize the training configuration, though as we discuss in App.~\ref{app:limitation}, this reflects real-world LLM development cycles.

Several promising directions remain. Can existing scaling law formulas be directly incorporated into GP priors to further improve optimization efficiency? Can the performance predictor be reused across different tasks to amortize its training cost? Addressing these questions would further reduce the computational demands of \JointSLBO{}, bringing principled black-box co-optimization closer to the practical requirements of LLM development.

\begin{ack}
This research/project is supported by the National Research Foundation, Singapore
under its AI Singapore Programme (AISG Award No: AISG3-PhD-2025-08-061). Zhiliang Chen is supported by the I2R at Agency for Science, Technology and Research (A*STAR). This research is supported by the National Research Foundation, Singapore under its National Large Language Models Funding Initiative (AISG Award No: AISG-NMLP-2024-001).

\end{ack}




\newpage
\bibliography{example_paper}

@article{bergsma2025powerlinesscalinglaws,
      title={Power Lines: Scaling Laws for Weight Decay and Batch Size in LLM Pre-training}, 
      author={Shane Bergsma and Nolan Dey and Gurpreet Gosal and Gavia Gray and Daria Soboleva and Joel Hestness},
      year={2025},
      journal={arXiv:2505.13738},
}

@article{daulton2022multiobjectivebayesianoptimizationhighdimensional,
      title={Multi-Objective Bayesian Optimization over High-Dimensional Search Spaces}, 
      author={Samuel Daulton and David Eriksson and Maximilian Balandat and Eytan Bakshy},
      year={2022},
      journal={arXiv:2109.10964}, 
}

@article{li2025predictablescaleiifarseer,
      title={Predictable Scale: Part II, Farseer: A Refined Scaling Law in Large Language Models}, 
      author={Houyi Li and Wenzhen Zheng and Qiufeng Wang and Zhenyu Ding and Haoying Wang and Zili Wang and Shijie Xuyang and Ning Ding and Shuigeng Zhou and Xiangyu Zhang and Daxin Jiang},
      year={2025},
      journal={arXiv:2506.10972},
}

@article{wu2019practicalmultifidelitybayesianoptimization,
      title={Practical Multi-fidelity Bayesian Optimization for Hyperparameter Tuning}, 
      author={Jian Wu and Saul Toscano-Palmerin and Peter I. Frazier and Andrew Gordon Wilson},
      year={2019},
      journal={arXiv:1903.04703},
}

@article{frazier2018tutorialbayesianoptimization,
      title={A Tutorial on Bayesian Optimization}, 
      author={Peter I. Frazier},
      year={2018},
      journal={arXiv:1807.02811}
}

@article{shukor2025scalinglawsoptimaldata,
      title={Scaling Laws for Optimal Data Mixtures}, 
      author={Mustafa Shukor and Louis Bethune and Dan Busbridge and David Grangier and Enrico Fini and Alaaeldin El-Nouby and Pierre Ablin},
      year={2025},
      journal={arXiv:2507.09404},
}

@misc{eval-harness,
  author       = {Gao, Leo and Tow, Jonathan and Abbasi, Baber and Biderman, Stella and Black, Sid and DiPofi, Anthony and Foster, Charles and Golding, Laurence and Hsu, Jeffrey and Le Noac'h, Alain and Li, Haonan and McDonell, Kyle and Muennighoff, Niklas and Ociepa, Chris and Phang, Jason and Reynolds, Laria and Schoelkopf, Hailey and Skowron, Aviya and Sutawika, Lintang and Tang, Eric and Thite, Anish and Wang, Ben and Wang, Kevin and Zou, Andy},
  title        = {The Language Model Evaluation Harness},
  month        = 07,
  year         = 2024,
  publisher    = {Zenodo},
  version      = {v0.4.3},
  doi          = {10.5281/zenodo.12608602},
  url          = {https://zenodo.org/records/12608602}
}

@article{zhang2024scalingmeetsllmfinetuning,
      title={When Scaling Meets LLM Finetuning: The Effect of Data, Model and Finetuning Method}, 
      author={Biao Zhang and Zhongtao Liu and Colin Cherry and Orhan Firat},
      year={2024},
      journal={arXiv:2402.17193},
}

@article{hendrycks2021measuringmassivemultitasklanguage,
      title={Measuring Massive Multitask Language Understanding}, 
      author={Dan Hendrycks and Collin Burns and Steven Basart and Andy Zou and Mantas Mazeika and Dawn Song and Jacob Steinhardt},
      year={2021},
      journal={arXiv:2009.03300},
}

@article{wikitext-data,
      title={Pointer Sentinel Mixture Models}, 
      author={Stephen Merity and Caiming Xiong and James Bradbury and Richard Socher},
      year={2016},
      journal={arXiv:1609.07843},
}

@article{gsm8k,
      title={Training Verifiers to Solve Math Word Problems}, 
      author={Karl Cobbe and Vineet Kosaraju and Mohammad Bavarian and Mark Chen and Heewoo Jun and Lukasz Kaiser and Matthias Plappert and Jerry Tworek and Jacob Hilton and Reiichiro Nakano and Christopher Hesse and John Schulman},
      year={2021},
      journal={arXiv:2110.14168},
}

@article{pubmedqa,
      title={PubMedQA: A Dataset for Biomedical Research Question Answering}, 
      author={Qiao Jin and Bhuwan Dhingra and Zhengping Liu and William W. Cohen and Xinghua Lu},
      year={2019},
      journal={arXiv:1909.06146},
}

@inproceedings{sciq,
    title = "Crowdsourcing Multiple Choice Science Questions",
    author = "Welbl, Johannes  and
      Liu, Nelson F.  and
      Gardner, Matt",
    booktitle = "Workshop on Noisy User-generated Text",
    year = {2017},
}

@article{triviaQA,
      title={TriviaQA: A Large Scale Distantly Supervised Challenge Dataset for Reading Comprehension}, 
      author={Mandar Joshi and Eunsol Choi and Daniel S. Weld and Luke Zettlemoyer},
      year={2017},
      journal={arXiv:1705.03551},
}

@article{xie2023doremioptimizingdatamixtures,
      title={DoReMi: Optimizing Data Mixtures Speeds Up Language Model Pretraining}, 
      author={Sang Michael Xie and Hieu Pham and Xuanyi Dong and Nan Du and Hanxiao Liu and Yifeng Lu and Percy Liang and Quoc V. Le and Tengyu Ma and Adams Wei Yu},
      year={2023},
      journal={arXiv:2305.10429}
}

@article{hu2021lora,
  title={Lora: Low-rank adaptation of large language models},
  author={Hu, Edward J and Shen, Yelong and Wallis, Phillip and Allen-Zhu, Zeyuan and Li, Yuanzhi and Wang, Shean and Wang, Lu and Chen, Weizhu},
  journal={arXiv:2106.09685},
  year={2021}
}

@article{chen2025duet,
  title={DUET: Optimizing Training Data Mixtures via Feedback from Unseen Evaluation Tasks},
  author={Chen, Zhiliang and Lau, Gregory Kang Ruey and Foo, Chuan-Sheng and Low, Bryan Kian Hsiang},
  journal={arXiv:2502.00270},
  year={2025}
}

@inproceedings{tay2023bayesian_cost,
  title={Bayesian Optimization with Cost-varying Variable Subsets},
  author={Tay, Sebastian Shenghong and Foo, Chuan-Sheng and Urano, Daisuke and Leong, Richalynn and Low, Bryan Kian Hsiang},
  booktitle={Proc. NeurIPS},
  year={2023}
}

@article{ye2024datamixinglawsoptimizing,
      title={Data Mixing Laws: Optimizing Data Mixtures by Predicting Language Modeling Performance}, 
      author={Jiasheng Ye and Peiju Liu and Tianxiang Sun and Yunhua Zhou and Jun Zhan and Xipeng Qiu},
      year={2024},
      journal={arXiv:2403.16952},
}

@article{kaplan2020scalinglawsneurallanguage,
      title={Scaling Laws for Neural Language Models}, 
      author={Jared Kaplan and Sam McCandlish and Tom Henighan and Tom B. Brown and Benjamin Chess and Rewon Child and Scott Gray and Alec Radford and Jeffrey Wu and Dario Amodei},
      year={2020},
      journal={arXiv:2001.08361},
}

@article{hoffmann2022trainingcomputeoptimallargelanguage,
      title={Training Compute-Optimal Large Language Models}, 
      author={Jordan Hoffmann and Sebastian Borgeaud and Arthur Mensch and Elena Buchatskaya and Trevor Cai and Eliza Rutherford and Diego de Las Casas and Lisa Anne Hendricks and Johannes Welbl and Aidan Clark and Tom Hennigan and Eric Noland and Katie Millican and George van den Driessche and Bogdan Damoc and Aurelia Guy and Simon Osindero and Karen Simonyan and Erich Elsen and Jack W. Rae and Oriol Vinyals and Laurent Sifre},
      year={2022},
      journal={arXiv2203.15556},
}

@article{xia2024lessselectinginfluentialdata,
      title={LESS: Selecting Influential Data for Targeted Instruction Tuning}, 
      author={Mengzhou Xia and Sadhika Malladi and Suchin Gururangan and Sanjeev Arora and Danqi Chen},
      year={2024},
      journal={arXiv:2402.04333}
}

@article{koh2020understandingblackboxpredictionsinfluence,
      title={Understanding Black-box Predictions via Influence Functions}, 
      author={Pang Wei Koh and Percy Liang},
      year={2020},
      journal={arXiv:1703.04730},
}

@inproceedings{
zhang2025_diversity_data_selection,
title={Harnessing Diversity for Important Data Selection in Pretraining Large Language Models},
author={Chi Zhang and Huaping Zhong and Kuan Zhang and Chengliang Chai and Rui Wang and Xinlin Zhuang and Tianyi Bai and Qiu Jiantao and Lei Cao and Ju Fan and Ye Yuan and Guoren Wang and Conghui He},
booktitle={Proc. ICLR},
year={2025}
}

@article{wang2024diversity-log-det,
      title={Diversity Measurement and Subset Selection for Instruction Tuning Datasets}, 
      author={Peiqi Wang and Yikang Shen and Zhen Guo and Matthew Stallone and Yoon Kim and Polina Golland and Rameswar Panda},
      year={2024},
      journal={arXiv:2402.02318},
}

@article{raschka2020modelevaluationmodelselection,
      title={Model Evaluation, Model Selection, and Algorithm Selection in Machine Learning}, 
      author={Sebastian Raschka},
      year={2020},
      journal={arXiv:1811.12808}, 
}

@article{ai2_arc,
      title={Think you have Solved Question Answering? Try ARC, the AI2 Reasoning Challenge}, 
      author={Peter Clark and Isaac Cowhey and Oren Etzioni and Tushar Khot and Ashish Sabharwal and Carissa Schoenick and Oyvind Tafjord},
      year={2018},
      journal={arXiv:1803.05457},
}

@article{darts,
      title={DARTS: Differentiable Architecture Search}, 
      author={Hanxiao Liu and Karen Simonyan and Yiming Yang},
      year={2019},
      journal={arXiv:1806.09055},
}

@article{bananas,
      title={BANANAS: Bayesian Optimization with Neural Architectures for Neural Architecture Search}, 
      author={Colin White and Willie Neiswanger and Yash Savani},
      year={2020},
      journal={arXiv:1910.11858},
}

@article{sobol,
      title={Practical Batch Bayesian Optimization for Less Expensive Functions}, 
      author={Vu Nguyen and Sunil Gupta and Santu Rana and Cheng Li and Svetha Venkatesh},
      year={2018},
      journal={arXiv:1811.01466},
}

@article{he2024robustifyingboostingtrainingfreeneural,
      title={Robustifying and Boosting Training-Free Neural Architecture Search}, 
      author={Zhenfeng He and Yao Shu and Zhongxiang Dai and Bryan Kian Hsiang Low},
      year={2024},
      journal={arXiv:2403.07591},
}

@article{zhang2024autoloraautomaticallytuningmatrix,
      title={AutoLoRA: Automatically Tuning Matrix Ranks in Low-Rank Adaptation Based on Meta Learning}, 
      author={Ruiyi Zhang and Rushi Qiang and Sai Ashish Somayajula and Pengtao Xie},
      year={2024},
      journal={arXiv:2403.09113},
}

@inproceedings{bo-kernelized-bandits,
  title={On kernelized multi-armed bandits},
  author={Chowdhury, Sayak Ray and Gopalan, Aditya},
  booktitle={Proc. ICML},
  year={2017},
}

@inproceedings{
chen2024towardsautoai,
title={Towards Auto{AI}: Optimizing a Machine Learning System with Black-box and Differentiable Components},
author={Zhiliang Chen and Chuan-Sheng Foo and Bryan Kian Hsiang Low},
booktitle={Forty-first International Conference on Machine Learning},
year={2024},
url={https://openreview.net/forum?id=JVhUR8q27o}
}

@misc{chen2026efficientlyestimatingoptimalhyperparameter,
      title={Efficiently Estimating Optimal Hyperparameter Scaling Laws through Power-Law Entropy Search}, 
      author={Zhiliang Chen and Sebastian Ament and David Eriksson and Maximilian Balandat and Bryan Kian Hsiang Low and Eytan Bakshy and Jihao Andreas Lin},
      year={2026},
      eprint={2609.01431},
      archivePrefix={arXiv},
      primaryClass={cs.LG},
      url={https://arxiv.org/abs/2609.01431}, 
}

@inproceedings{verma-etal-2025-uncovering,
    title = "Uncovering Scaling Laws for Large Language Models via Inverse Problems",
    author = "Verma, Arun  and
      Wu, Zhaoxuan  and
      Zhou, Zijian  and
      Lin, Xiaoqiang  and
      Chen, Zhiliang  and
      Sim, Rachael Hwee Ling  and
      Qiao, Rui  and
      Wang, Jingtan  and
      Bui, Nhung  and
      Niu, Xinyuan  and
      Hu, Wenyang  and
      Lau, Gregory Kang Ruey  and
      Khoo, Zi-Yu  and
      Zhao, Zitong  and
      Xu, Xinyi  and
      Hemachandra, Apivich  and
      Ng, See-Kiong  and
      Low, Bryan Kian Hsiang",
    editor = "Christodoulopoulos, Christos  and
      Chakraborty, Tanmoy  and
      Rose, Carolyn  and
      Peng, Violet",
    booktitle = "Findings of the {A}ssociation for {C}omputational {L}inguistics: {EMNLP} 2025",
    month = nov,
    year = "2025",
    address = "Suzhou, China",
    publisher = "Association for Computational Linguistics",
    url = "https://aclanthology.org/2025.findings-emnlp.1373/",
    doi = "10.18653/v1/2025.findings-emnlp.1373",
    pages = "25197--25211",
    ISBN = "979-8-89176-335-7"
}

@article{truthfulQA,
      title={TruthfulQA: Measuring How Models Mimic Human Falsehoods}, 
      author={Stephanie Lin and Jacob Hilton and Owain Evans},
      year={2022},
      journal={arXiv:2109.07958},
}

@inproceedings{attention,
 author = {Vaswani, Ashish and Shazeer, Noam and Parmar, Niki and Uszkoreit, Jakob and Jones, Llion and Gomez, Aidan N and Kaiser, \L ukasz and Polosukhin, Illia},
title = 	{Attention is all you need},
 booktitle = {Proc. Neurips},
 year = {2017}
}

@article{wang2024helpfulharmfuldatafinetuningfree,
      title={Helpful or Harmful Data? Fine-tuning-free Shapley Attribution for Explaining Language Model Predictions}, 
      author={Jingtan Wang and Xiaoqiang Lin and Rui Qiao and Chuan-Sheng Foo and Bryan Kian Hsiang Low},
      year={2024},
      journal={arXiv:2406.04606},
}

@article{scaling_law_performance,
      title={Scaling Laws for Predicting Downstream Performance in LLMs}, 
      author={Yangyi Chen and Binxuan Huang and Yifan Gao and Zhengyang Wang and Jingfeng Yang and Heng Ji},
      year={2025},
      journal={arXiv:2410.08527},
      archivePrefix={arXiv},
}

@article{scaling_law_performance_2,
      title={Performance Law of Large Language Models}, 
      author={Chuhan Wu and Ruiming Tang},
      year={2024},
      journal={arXiv:2408.09895},
}

@article{talmor2019commonsenseqaquestionansweringchallenge,
      title={CommonsenseQA: A Question Answering Challenge Targeting Commonsense Knowledge}, 
      author={Alon Talmor and Jonathan Herzig and Nicholas Lourie and Jonathan Berant},
      year={2019},
      journal={arXiv:1811.00937}
}

@book{gp-for-ml,
  title={Gaussian processes for machine learning},
  author={Williams, Christopher KI and Rasmussen, Carl Edward},
  volume={2},
  year={2006},
  publisher={MIT press Cambridge, MA}
}

@inproceedings{bo-gp-ucb-10,
author = {Srinivas, Niranjan and Krause, Andreas and Kakade, Sham and Seeger, Matthias},
title = {Gaussian Process Optimization in the Bandit Setting: No Regret and Experimental Design},
year = {2010},
booktitle = {Proc. ICML},
}

@InProceedings{bos-zhongxiang,
  title = 	 {{B}ayesian Optimization Meets {B}ayesian Optimal Stopping},
  author =       {Dai, Zhongxiang and Yu, Haibin and Low, Bryan Kian Hsiang and Jaillet, Patrick},
  booktitle = 	 {Proc. ICML},
  year = 	 {2019}
}

@article{grattafiori2024llama,
  title={The llama 3 herd of models},
  author={Grattafiori, Aaron and Dubey, Abhimanyu and Jauhri, Abhinav and Pandey, Abhinav and Kadian, Abhishek and Al-Dahle, Ahmad and Letman, Aiesha and Mathur, Akhil and Schelten, Alan and Vaughan, Alex and others},
  journal={arXiv preprint arXiv:2407.21783},
  year={2024}
}

@article{data_selection_1,
      title={Data Selection for Language Models via Importance Resampling}, 
      author={Sang Michael Xie and Shibani Santurkar and Tengyu Ma and Percy Liang},
      year={2023},
      journal={arXiv:2302.03169},
}

@article{chen2025aioliunifiedoptimizationframework,
      title={Aioli: A Unified Optimization Framework for Language Model Data Mixing}, 
      author={Mayee F. Chen and Michael Y. Hu and Nicholas Lourie and Kyunghyun Cho and Christopher Ré},
      year={2025},
      journal={arXiv:2411.05735},
}

@article{liu2025regmixdatamixtureregression,
      title={RegMix: Data Mixture as Regression for Language Model Pre-training}, 
      author={Qian Liu and Xiaosen Zheng and Niklas Muennighoff and Guangtao Zeng and Longxu Dou and Tianyu Pang and Jing Jiang and Min Lin},
      year={2025},
      journal={arXiv:2407.01492},
}

@article{lee2025costsensitivefreezethawbayesianoptimization,
      title={Cost-Sensitive Freeze-thaw Bayesian Optimization for Efficient Hyperparameter Tuning}, 
      author={Dong Bok Lee and Aoxuan Silvia Zhang and Byungjoo Kim and Junhyeon Park and Steven Adriaensen and Juho Lee and Sung Ju Hwang and Hae Beom Lee},
      year={2025},
      journal={arXiv:2510.21379},
}

@article{swersky2014freezethawbayesianoptimization,
      title={Freeze-Thaw Bayesian Optimization}, 
      author={Kevin Swersky and Jasper Snoek and Ryan Prescott Adams},
      year={2014},
      journal={1406.3896}
}

@misc{xie2025chameleonflexibledatamixingframework,
      title={Chameleon: A Flexible Data-mixing Framework for Language Model Pretraining and Finetuning}, 
      author={Wanyun Xie and Francesco Tonin and Volkan Cevher},
      year={2025},
      journal={arXiv:2505.24844}
}

@inproceedings{amortize_BO_1,
  title = 	 {Amortized Bayesian Optimization over Discrete Spaces},
  author =       {Swersky, Kevin and Rubanova, Yulia and Dohan, David and Murphy, Kevin},
  booktitle = {Proc. UAI},
year = {2020}
}

@misc{yen2025datamixtureoptimizationmultifidelity,
      title={Data Mixture Optimization: A Multi-fidelity Multi-scale Bayesian Framework}, 
      author={Thomson Yen and Andrew Wei Tung Siah and Haozhe Chen and Tianyi Peng and Daniel Guetta and Hongseok Namkoong},
      year={2025},
      journal={arXiv:2503.21023},
}
\bibliographystyle{plainnat}

\newpage
\appendix

\section{Impact statement}
\JointSLBO{} reduces the computational cost of fine-tuning LLMs by eliminating the need for full training runs at every optimization iteration. The primary societal benefit is democratization, as by making principled LLM optimization more computationally accessible, our method lowers the barrier for smaller research groups and organizations to develop high-quality, task-specific models without requiring large-scale compute infrastructure. We also show empirically that \JointSLBO{} improves the overall performance for LLMs in the sequential feedback setting.

\section{Additional related works on multi-fidelity BO}\label{app:relation-to-mf-bo}

Multi-fidelity BO methods typically offer the option to query a black-box function at either high or low fidelity. High-fidelity queries faithfully evaluate the true objective function but are more expensive, while low-fidelity queries provide cheaper but approximate estimates. Multi-fidelity BO incorporates fidelity as an \textit{additional input dimension} and uses a cost-aware acquisition function \citep{wu2019practicalmultifidelitybayesianoptimization} to determine which fidelity to query at every iteration.

In our setting, observing LLM performance after a fraction of the training steps can be considered a low-fidelity evaluation, while observing the LLM performance after full training can be viewed as a high-fidelity evaluation.  
The key distinction between \JointSLBO{} and standard multi-fidelity BO lies in how fidelities are modeled. Rather than jointly modeling low and high-fidelity observations with a GP (by treating fidelity as an additional input dimension), \JointSLBO{} explicitly models the relationship between low and high-fidelity observations using an expressive neural network --- the performance predictor --- which is \textit{decoupled from the GP surrogate}. This decoupling is important, as a GP is often not expressive enough to jointly model the relationship between fidelities, which explains why multi-fidelity BO methods underperform in our experiments. The same limitation applies to early-stopping BO~\citep{bos-zhongxiang}, which relies on a GP to model across fidelities. Furthermore, most multi-fidelity BO works do not analyze the optimal number of queries per fidelity. \JointSLBO{} addresses this by providing theoretical guidelines for choosing an optimal number of high-fidelity observations to query (Sec.~\ref{subsec:convergence-theorem}).

A related line of work, freeze-thaw BO~\citep{swersky2014freezethawbayesianoptimization,lee2025costsensitivefreezethawbayesianoptimization}, decides at each iteration whether to continue training an existing configuration or to explore a new one. We find this approach ill-suited to our setting for two reasons. First, freeze-thaw methods require storing multiple copies of fine-tuned LLMs simultaneously, which is memory-intensive at the current scale of LLMs. Second, they typically model performance improvement with an exponential decay kernel, which assumes monotonically improving performance with more training. However, Fig.~\ref{fig:scaling-law} shows that in our setting, the change in LLM performance can \textit{decrease} with more training steps when the chosen data mixture poorly matches the downstream task. As such, the exponential decay kernel introduced in these works is unable to capture this behavior and is ill-suited to our setting.

\section{Downstream task metrics and sequential feedback setting} 
\label{app:feedback-info}

\JointSLBO{} uses BO as its algorithmic backbone, leveraging iterative downstream task metric feedback to improve the training configuration. Crucially, \JointSLBO{} is agnostic to the form of this feedback. In the most practical deployment scenario, the metric could reflect real-world user behavior, such as user ratings or time spent on a chatbot, aggregated over a set of users.

In our experiments, we consider three forms of downstream task metrics to demonstrate \JointSLBO{}' versatility: downstream evaluation loss, performance on a specific QA benchmark, and aggregated performance across multiple QA benchmarks (i.e., a multi-task setting). Importantly, the specific benchmark used for evaluation is never revealed to our optimization algorithm (only its numerical output is observed), mirroring real-world conditions where the evaluation data distribution is unknown~\citep{chen2025duet}.

\section{Additional background on Bayesian optimization}\label{app:BO}

At every BO iteration, the BO algorithm proposes the next configuration $x_{t+1}$
as the configuration which maximizes some acquisition function, such as the
\emph{upper confidence bound} (UCB) \citep{bo-gp-ucb-10}, given by $x_{t+1} = \argmax_x \mu_t(x) + \beta_{t+1}\sigma_t(x)$, where $\beta_{t+1}$ is an exploration parameter.

A common kernel choice is the \textit{squared exponential} (SE) kernel $\kappa(x,x')\triangleq \exp(-\lVert x-x'\rVert_2^2/(2m^2))$ with a  \textit{length-scale} hyperparameter $m$ that can be learned via maximum likelihood estimation from observations.

\section{Proof of Theorem \ref{thm:convergence-main}}\label{app:proof-BO-convergence}

\convergence*

\begin{proof}

To begin, recall that we are trying to maximize our LLM performance, a black-box function $\mathcal{L}(\theta_{\mathcal{X}, \mathcal{M}, B})$ (Sec.~\ref{sec:prelim}). Using our scaling law predictor (Sec.~\ref{sec:scaling-law-predictor}), we instead train our LLM for $\Bsmall$ training steps (or time) and observe $\mathcal{L}(\theta_{\mathcal{X}, \mathcal{M}, \Bsmall})$. We then use the performance predictor to produce $\mathcal{F}(\mathcal{L}(\theta_{\mathcal{X}, \mathcal{M}, \Bsmall}))$, estimating the LLM performance if we trained it fully for $B$ training steps. Since we are predicting the LLM performance, our model prediction is noisy with $\mathcal{F}(\mathcal{L}(\theta_{\mathcal{X}, \mathcal{M}, \Bsmall})) = {\mathcal{L}}(\theta_{\mathcal{X}, \mathcal{M}, B}) + \epsilon$. Hence, we only have access to a noisy estimate of our black-box function: ${\mathcal{L}}(\theta_{\mathcal{X}, \mathcal{M}, B}) + \epsilon$. The prediction error is also $R$-sub-Gaussian as suggested from our empirical findings (Fig.~\ref{fig:prediction-error-f-samples-N}). As such, this formulation is in line with the BO algorithmic framework introduced in Sec.~\ref{sec:BO}, where we obtain noisy observations of the true underlying function.

Next, we aim to draw the relationship between the choice of $N$ and our predictor's error $\epsilon$ and consequently, \JointSLBO{}' regret. First, we present the following lemma from \citep{bo-kernelized-bandits}:

\begin{lemma}
\label{lemma:concentration}
    Let $||f||_{\kappa}=\sqrt{	\langle f,f \rangle_\kappa} \leq \mathcal{B}$. Also, assume that the observation noise associated with each BO iteration is $R$-sub-Gaussian with $R>0$. Then with probability at least $1-\delta$, the following holds for BO iteration $t \leq T$:
    \begin{equation}
        |\mu_t(x)-f(x)| \leq \left(\mathcal{B} + R \sqrt{2(\gamma_t + 1 + \ln(1/\delta))}\right)\sigma_t(x)
    \end{equation}
    \label{eq:lemma-concentration}
    where $\gamma_{t}$ is the maximum information gain after $t$ observations and $\mu_t(x), \sigma_t^2(x)$ are mean and variance of posterior distribution of GP defined in Eq.~\ref{gp:posterior} with $\zeta=1+2/T$.
\end{lemma}

In our setting, set $f=\mathcal{L}$ (our LLM performance after fine-tuning) and $x=\mathcal{X},\mathcal{M}$ (our LLM training configuration). This lemma indicates that our estimated mean $\mu_t(x)$ of our performance landscape from our fitted GP over historical observations of LLM performance deviates from the true LLM performance $f(x)={\mathcal{L}}(\theta_{\mathcal{X}, \mathcal{M}, B})$ by at most the right-hand side term in Lemma~\ref{eq:lemma-concentration}.

We are now ready to prove Theorem~\ref{thm:convergence-main}. First, we observe that the next LLM training configuration $x_t$ at each BO iteration $t$ is chosen via the IGP-UCB acquisition function (i.e., $x_t = \argmax_{x} \mu_{t-1}(x) + \beta_t \sigma_{t-1}(x)$ and $\beta_{t} = \mathcal{B} + R \sqrt{2(\gamma_{t-1}+1+\ln(1/\delta))}$ where the observation noise associated with each BO iteration is $R$-sub-Gaussian). Thus, we can see that at each iteration $t \geq 1$, we have $\mu_{t-1}(x_t) + \beta_t \sigma_{t-1}(x_t) \geq \mu_{t-1}(x^*) + \beta_t \sigma_{t-1}(x^*)$. It then follows that for all $t \geq 1$ and with probability at least $1-\delta$,
\begin{equation}
    \begin{split}
        f(x^*)-f(x_t) &\stackrel{(1)}{\leq} \beta_t \sigma_{t-1}(x_t) + \mu_{t-1}(x_t) - f(x_t) \\
        &\stackrel{(2)}{\leq} \beta_t \sigma_{t-1}(x_t) + \mu_{t-1}(x_t) + (\beta_t\sigma_{t-1}(x_t) - \mu_{t-1}(x_t)) \\
        &\leq 2\beta_t \sigma_{t-1}(x_t)
    \end{split}
    \label{eq:bounded-f}
\end{equation}
where $\stackrel{(1)}{\leq}$ uses the fact that via Lemma~\ref{lemma:concentration} and our acquisition function, $f(x^*) \leq \beta_t \sigma_{t-1}(x^*) + \mu_{t-1}(x^*) \leq \beta_t \sigma_{t-1}(x_t) + \mu_{t-1}(x_t)$ and $\stackrel{(2)}{\leq}$ once again uses Lemma~\ref{lemma:concentration}.

Next, we adjust the total number of BO iterations in our problem setting, depending on the chosen $N$. As mentioned in the analysis of Sec.~\ref{subsec:performance-tradeof}, under an overall optimization budget of $C$, collecting $N$ initial full-training runs leaves enough budget for $\floor{\frac{C-NB}{\Bsmall}}$ BO iterations. Hence, we analyze how \JointSLBO{}' cumulative regret grows for up to $T=\floor{\frac{C-NB}{\Bsmall}}$ iterations. Using Eq.~\ref{eq:bounded-f}, we can bound the cumulative regret by
\begin{equation}
\sum_{t=1}^{T} r_t = \sum_{t=1}^{T} (f(x^*) - f(x_t) ) \leq 2\sum_{t=1}^{T} \beta_t \sigma_{t-1}(x_t)\
\end{equation}
where $r_t$ is the per-iteration regret.

Since we know that $\displaystyle \sum_{t=1}^{T} \sigma_{t-1}(x_t) = \mathcal{O}(\sqrt{{T}\gamma_{T}})$ and used $\beta_t = \mathcal{B} + R \sqrt{2(\gamma_{t-1}+1+\ln(1/\delta))}$, the cumulative regret in Theorem~\ref{thm:convergence-main} can be written as
\begin{align}
    \mathcal{R}_{T} &= \sum_{t=1}^{T} r_t \\
    &\leq 2 \sum_{t=1}^{T} \beta_t \sigma_{t-1}(x_t) \\
    &\leq 2 \mathcal{O}(\sqrt{{T}\gamma_{T}}) (\mathcal{B} + R \sqrt{2(\gamma_{T}+1+\ln(1/\delta))}) \\
    &= \mathcal{O}\left(\mathcal{B}\sqrt{{T} \gamma_{T}} + R\sqrt{{T}} \sqrt{\gamma_{T}^2 + \gamma_{T} \ln(1/\delta)}\right) \\
    &\stackrel{(1)}{=} \mathcal{O}\left(\mathcal{B}\sqrt{{\floor*{\frac{C-NB}{\Bsmall}}} \gamma_{T}} + R\sqrt{{\floor*{\frac{C-NB}{\Bsmall}}}} \sqrt{\gamma_{T}^2 + \gamma_{T} \ln(1/\delta)}\right) \\
    &\stackrel{(2)}{=} \mathcal{O}\left(\mathcal{B}\sqrt{{\floor*{\frac{C-NB}{\Bsmall}}} \gamma_{T}} + \frac{k}{\sqrt{N}}\sqrt{{\floor*{\frac{C-NB}{\Bsmall}}}} \sqrt{\gamma_{T}^2 + \gamma_{T} \ln(1/\delta)}\right), \label{eq:cumulative-regret-final}
\end{align}
where $\stackrel{(1)}{=}$ uses the fact that if we collected $N$ full-training runs initially, the number of BO iterations available is $T=\floor*{\frac{C-NB}{\Bsmall}}$ and $\stackrel{(2)}{=}$ uses \textbf{Assumption 1} that our prediction error is $R$-sub-Gaussian with $R=\frac{k}{\sqrt{N}}$.

Lastly, because different numbers of initial full-training runs $N$ will influence the number of remaining BO iterations, we derive the average regret w.r.t. the number of BO iterations by dividing the cumulative regret bounds obtained in Eq.~\eqref{eq:cumulative-regret-final} by $\floor*{\frac{C-NB}{\Bsmall}}$ throughout:

\begin{align}
    \frac{\mathcal{R}_{T}}{T} &\stackrel{(1)}{=}
    \mathcal{O}\left(\mathcal{B}\sqrt{{\frac{\gamma_{T} \Bsmall}{C-NB-\Bsmall}}} + \frac{k}{\sqrt{N}}\sqrt{\frac{\Bsmall}{C-NB-\Bsmall}} \sqrt{\gamma_{T}^2 + \gamma_{T} \ln(1/\delta)}\right) \\
    &= \mathcal{O}\left(\left(\sqrt{{\frac{\Bsmall}{C-\textcolor{black}{N}B-\Bsmall}}}\right)\left(\mathcal{B}\sqrt{\gamma_{T}} + \frac{k}{\sqrt{\textcolor{black}{N}}} \sqrt{\gamma_{T}^2 + \gamma_{T} \ln(1/\delta)}\right)\right)
\end{align}
where $\stackrel{(1)}{=}$ divides the cumulative regret in Eq.~\eqref{eq:cumulative-regret-final} throughout by $T=\floor*{\frac{C-NB}{\Bsmall}}$ and uses the fact that $\frac{1}{\floor*{a/b}} \leq \frac{b}{a-b}$ when $a > b > 0$ (which holds here since $C - NB > \Bsmall$ in our setting).

One interesting point to note is that our algorithm incurs some regret when collecting $N$ random training configurations initially. However, since our black-box function is bounded (under the SE-kernel) and $N$ is not assumed to be dependent on the total optimization budget $T$, then the average regret incurred when collecting training configurations initially can be removed in the big-$O$ notation.
\end{proof}

\subsection{Comparison of regret bound in vanilla BO}\label{app:comparison-vanilla-BO}

We show the conditions in which the average regret upper-bound of \JointSLBO{} is smaller than that of the vanilla BO case. To simplify notation, let $T$ be the number of BO iterations available in the vanilla BO case. Recall that in normal BO (without any performance predictor), the cumulative regret is \citep{chen2024towardsautoai,bo-kernelized-bandits}
\begin{equation}
    \mathcal{O}\left(\mathcal{B}\sqrt{{T} \gamma_{T}} + R\sqrt{{T}} \sqrt{\gamma_{T}^2 + \gamma_{T} \ln(1/\delta)}\right).
\end{equation}
Dividing by $T$, we obtain an average regret of:
\begin{equation}
    \mathcal{O}\left(\frac{\mathcal{B}\sqrt{\gamma_{T}}}{\sqrt{T}} + \frac{R \sqrt{\gamma_{T}^2 + \gamma_{T} \ln(1/\delta)}}{\sqrt{T}}\right).
\end{equation}

Consider that in our case, \JointSLBO{} allocates a constant number of iterations (e.g., $N \ll T$) to learn the performance predictor, which yields noisy observations with error distribution that is $R_1$-sub-Gaussian, where $R_1 > R$. This means the predictor does not recover the original performance landscape perfectly. By doing so, each BO iteration consumes 10\% of the original resources (since we only need to fine-tune an LLM to 10\% of the original number of training steps) and hence, we can perform $10(T-N)$ more BO iterations. Assuming that the average regret of the first $N$ iterations where we make random queries is constant (we can ignore it in the asymptotic bound), this implies the average regret bound of \JointSLBO{} is:

\begin{equation}
    \mathcal{O}\left(\frac{\mathcal{B}\sqrt{\gamma_{T}}}{\sqrt{10(T-N)}} + \frac{R_1 \sqrt{\gamma_{T}^2 + \gamma_{T} \ln(1/\delta)}}{\sqrt{10(T-N)}}\right).
\end{equation}
Note that this is essentially the same bound as Theorem~\ref{thm:convergence-main} but without the complication of deriving the expression for the number of BO acquisition steps available under a limited optimization budget. Comparing Eq.(16) and Eq.(17), we see that the average regret bound of \JointSLBO{} is smaller than vanilla BO with probability at least $1-\delta$ if and only if:
\begin{equation}
    \left(\mathcal{B}\sqrt{\gamma_T} + \textcolor{red}{R_1}\sqrt{\gamma_{T}^2 + \gamma_{T} \ln(1/\delta)}\right) < \left(\mathcal{B}\sqrt{\gamma_T} + \textcolor{red}{R}\sqrt{\gamma_{T}^2 + \gamma_{T} \ln(1/\delta)}\right) \textcolor{red}{\sqrt{\frac{10(T-N)}{T}}}.
\end{equation}
The differences between the left and right side are highlighted in \textcolor{red}{red} and this inequality admits a neat form that is easy to interpret: $R_1$ is larger than $R$ (i.e., using a predictor gives noisier observations), but as long as $\textcolor{red}{\sqrt{\frac{10(T-N)}{T}}}$ is much larger (i.e., we can run more BO iterations; it is almost 10 times larger in this case), the average regret of \JointSLBO{} will be smaller than that found in the vanilla BO case.

\section{Computational complexity}
\label{app:computational-complexity}

The computational complexity of the BO algorithm is well studied~\citep{bo-gp-ucb-10}; a detailed summary is provided in Appendix F of~\citet{chen2025duet}. Briefly, the dominant computational cost scales cubically with the number of iterations $T$, i.e., $\mathcal{O}(T^3)$, due to kernel matrix factorization (e.g., via Cholesky decomposition). The dependence on problem dimension $n$ arises primarily from acquisition function optimization, which incurs $\mathcal{O}(c\,nT^2)$ per iteration, where $c$ is the number of optimization restarts. In typical regimes where $n \ll T$, the cubic dependence on $T$ dominates, yielding an overall complexity of $\mathcal{O}(T^3)$ with only linear dependence on $n$.

Beyond the BO overhead, each iteration requires training an LLM, which is the dominant practical cost. Naively, $T$ BO iterations require $T$ full LLM training runs. However, this cost need not be front-loaded. In realistic deployment scenarios, the LLM can be rolled out to collect metric feedback over an extended period, as is standard practice in real-world LLM development (e.g., ChatGPT and Gemini are both rolled out and improved over months or years). 

\section{Practical design choices}
\label{app:practical-consideration}
We discuss several practical design choices involved in \JointSLBO{}. First, we use a mixed kernel for the GP that combines a Hamming-distance kernel with an SE kernel, enabling more principled handling of the mixed discrete-continuous input space (data mixtures are continuous variables while LoRA rank and layer are discrete variables). Second, data from prior training runs (earlier experiments or published works) can be used to learn the predictor at no additional budget cost, further improving prediction accuracy (Fig.~\ref{fig:prediction-error-f-samples-N}). Lastly, we choose the number of training steps as the budget constraint in our problem formulation to keep the number of data samples used in the BO loop constant. However, \JointSLBO{}' budget is a generalizable constraint, and other constraint choices such as total number of training tokens or floating-point operations (FLOPs) can be considered with little change to our formulation, since the number of data points and model size is approximately constant in our setting (the number of fine-tuning parameters is small relative to the model size).

\begin{table}[t]
\centering
\begin{tabular}{lccc}
\toprule
\textbf{Kernel} & \textbf{MMLU} & \textbf{TriviaQA} & \textbf{GSM8K} \\
\midrule
Hamming + RBF (SE2) & $74.4 \pm 0.4$ & $72.1 \pm 0.4$ & $86.0 \pm 0.7$ \\
RBF (SE2) & $74.5 \pm 0.3$ & $72.0 \pm 1.1$ & $86.4 \pm 1.2$ \\
Mat\'ern & $73.7 \pm 0.8$ & $71.0 \pm 0.6$ & $85.6 \pm 0.7$ \\
\bottomrule
\end{tabular}
\vspace{2mm}
\caption{Performance across different kernels on MMLU, TriviaQA, and GSM8K.}
\label{tab:kernel_results}
\end{table}

\section{Further discussion of limitations}\label{app:limitation}

One limitation of \JointSLBO{} is that it requires several rounds of feedback to iteratively optimize an LLM's training configurations. However, this characteristic is found in most black-box optimization settings (e.g., BO \cite{bo-gp-ucb-10}) and is not unique to our work. In addition, we view the process of obtaining feedback as a natural part of the LLM deployment life-cycle. For instance, many commercial LLMs, such as ChatGPT, Gemini or DeepSeek, are iteratively improved over 
many deployment iterations by gathering feedback from users (in fact, it would not be surprising if such iterative optimization approaches are already employed in practice under-the-hood).

In addition, in our setting, obtaining feedback is the only way we can optimize the LLM training configurations. Other methods, such as data selection, cannot exploit feedback and hence are suboptimal in our setting.

\section{Additional experimental details}
\label{app:exp-details}
\textbf{Search space.} The data configuration $\mathcal{X}$ consists of $9$ dimensions representing the mixing ratio (a probability simplex) over $9$ training data domains. The model configuration $\mathcal{M}$ consists of $10$ dimensions:
\begin{enumerate}
    \item Number of LLM layers to apply LoRA to $\in [1,31]$ (varies by LLM architecture),
    \item Whether to apply LoRA to front or rear layers (binary),
    \item Whether to apply LoRA to $Q$-projection (binary),
    \item Whether to apply LoRA to $K$-projection (binary),
    \item Whether to apply LoRA to $V$-projection (binary),
    \item Whether to apply LoRA to MLP-up-projection (binary),
    \item Whether to apply LoRA to MLP-down-projection (binary),
    \item LoRA rank $\in [1,256]$,
    \item LoRA dropout $\in [0,1]$, and
    \item LoRA $\alpha$ $\in [1,500]$.
\end{enumerate}

\textbf{Default configurations.} For baselines that do not optimize a particular component, we adopt the following default fine-tuning configuration: (a) LoRA applied to all LLM layers, (b) LoRA applied to all $Q$-, $V$-, $K$-, MLP-up-, MLP-down-projection, (c) rank $= 16$, (d) dropout $= 0.1$, (e) $\alpha = 16$, and (f) uniform data mixture. These defaults are consistent with standard configurations used in existing works and off-the-shelf fine-tuning frameworks. In our experiments, we use $8$-shot prompting with chain-of-thought reasoning, with $\Bsmall=100$ training steps, $B=1000$ training steps and optimization budget $C=50000$ training steps.

\textbf{Performance predictor.} To learn the performance predictor (Sec.~\ref{sec:scaling-law-predictor}), we sample a random Sobol sequence \citep{sobol} of $30$ LLM training configurations and fine-tune the LLM for $B=1000$ training steps (approximately 1.5 training epochs). During training, we collect the LLM evaluation loss or performance at intervals of $25$ training steps, up to $\Bsmall=100$ training steps. We construct the neural network input from observations at steps $25, 50, 75, 100$, together with the data and model configurations. The performance at $B=1000$ steps serves as the neural network's prediction target. We fit a densely-connected neural network $\mathcal{F}$ with $3$ layers of width $64$ on this dataset.

\textbf{BO setup.} The $N=30$ full training observations are also used to warm-start the GP surrogate, leaving a remaining budget of $\frac{50000 - 30\times1000}{100}=200$ BO iterations. We use a mixed kernel combining a Hamming-distance kernel and an SE kernel for the GP, implemented via the \texttt{BoTorch} library. Kernel hyperparameters are estimated via maximum likelihood at the end of every iteration. All training configuration parameters are normalized to $[0,1]$ before fitting the GP. Experiments were run on 4 \texttt{H200} GPUs and a training batch size of 64.

\section{Implementation details of baselines}\label{app:baseline-details}

\textbf{LESS} \citep{xia2024lessselectinginfluentialdata}. We pass each training data point from each data domain into the LLM to construct a gradient store \citep{xia2024lessselectinginfluentialdata}. Since we do not know the exact evaluation task (just like the setting considered in \citet{chen2025duet}), we form a validation mixture consisting of uniformly sampled data from each data domain, before retrieving a subset of $10000$ data points from the gradient store with the highest feature similarity with the evaluation dataset. We follow the implementation given in \url{https://github.com/princeton-nlp/LESS}.

\textbf{DoReMi} \citep{xie2023doremioptimizingdatamixtures}. We use DoReMi to optimize the data mixing ratio across the training data domains, according to the implementation given in \url{https://github.com/sangmichaelxie/doremi}. Again, since we do not have direct access to the evaluation task (as DoReMi requires an explicit validation dataset), we form the validation dataset using a uniform validation mixture similar to LESS (see above).

\textbf{IF} \citep{koh2020understandingblackboxpredictionsinfluence}. We compute the influence scores of each data point (following the implementation in \url{https://github.com/alstonlo/torch-influence}) w.r.t. a uniform validation mixture constructed from the training data domains (see above). Then, we retrieve the top $10000$ data points with the highest influence scores.

\textbf{Diversity} \citep{wang2024diversity-log-det}. We first project every data point into a continuous semantic space using an off-the-shelf embedding model. Then, we retrieve the data subset with the highest log-determinant value, which is derived by arranging data points' embeddings in a matrix and computing the determinant. There is a greedy implementation \citep{wang2024diversity-log-det} that improves the computational efficiency.

\textbf{DARTS} \citep{darts}. We optimize the same model configurations detailed in App.~\ref{app:exp-details} by introducing a differentiable parameter for each model component (e.g., LoRA rank). For example, LoRA rank is controlled by a single differentiable parameter, while LoRA layer selection is controlled by a softmax layer with weights, allowing us to drop certain LoRA layers (if the weight is decreased to 0 after update). We alternate between freezing the differentiable parameters that control the model configuration and model parameters that control the LLM behavior every 5 training steps. We optimize these objectives in an alternating fashion, switching between updating configuration parameters and model parameters every 5 steps.

\textbf{AutoLoRA} \citep{zhang2024autoloraautomaticallytuningmatrix}. We follow the implementation in \url{https://github.com/ruz048/AutoLoRA}, which tunes the LoRA rank for each layer.

\textbf{RoBoT} \citep{he2024robustifyingboostingtrainingfreeneural}. We utilize the training-free proxies given in \citep{he2024robustifyingboostingtrainingfreeneural} and perform BO over the weights assigned to each training-free proxy. With the optimized weights, we then iterate through the model configurations to find the one with the best (weighted) training-free measure.

\textbf{MF-KG} \citep{wu2019practicalmultifidelitybayesianoptimization}. The implementation is reproduced from \url{https://botorch.org/docs/tutorials/multi_fidelity_bo/}.

\textbf{MF-UCB}. We use a cost-aware UCB acquisition function that follows the framework in \citet{yen2025datamixtureoptimizationmultifidelity}.

\section{Computational cost of \JointSLBO{} vs.~baselines}\label{app:compute-cost}

\textbf{Quantitative comparison of wall-clock hours.} \JointSLBO{} is run for $50000$ training steps in total. This equates to approximately $50000$ seconds of wall-clock time ($13.3$ hours) using a \texttt{H200} GPU. All model selection methods \citep{he2024robustifyingboostingtrainingfreeneural,darts} used in our paper are iterative in nature and require repeated fine-tuning of LLMs. Equal time budget was allocated to both \JointSLBO{} and all baseline methods for a fair comparison. We found that given equal computation time, \JointSLBO{} achieves better performance than all model selection methods. For data optimization methods (LESS, DoReMi, IF, Diversity), we record their wall-clock runtime in Table~\ref{table:computation-time}. In general, we found data optimization methods to be computationally more expensive than \JointSLBO{}.

\begin{table}[h]
  \centering
  \caption{Wall-clock runtime comparison of data selection methods vs.~\JointSLBO{}.}
  \begin{tabular}{lc}
    \toprule
    \bfseries Method & \bfseries Time (hours) \\
    \midrule
    LESS       & $16.3$ \\
    DoReMi     & $18.5$ \\
    IF         & $52$ \\
    Diversity  & $122$ \\
    \JointSLBO{} & $\cellcolor{darkgreen!50}\mathbf{\underline{13.3}}$ \\
    \bottomrule
  \end{tabular}
  \label{table:computation-time}
\end{table}

\section{Optimal LLM training configurations found by \JointSLBO{} compared to baselines}\label{app:qualitative-config-found-analysis}

We show some of the optimal LLM training configurations found by \JointSLBO{} vs.~other baselines. We divided the configurations into two tables detailing the best data (Table~\ref{table-optimal-ratios}) and model (Table~\ref{table-optimal-model-config}) configurations found for the \textbf{GSM$8$K} evaluation task.

\begin{table}[h]
\centering
\setlength{\tabcolsep}{4pt}
\caption{Optimal data mixing ratio found by \JointSLBO{} vs.~other baselines. The columns denote the ratio allocated to each training domain.}\label{table-optimal-ratios}
\begin{tabular}{c *{9}{c}}
\toprule
 & CQA & GSM$8$K  & PubmedQA & SciQ & TrivQA & TruthQA & Wiki & MMLU & ARC \\
\midrule
\JointSLBO{} & 0 & 0.34 &  0 & 0.10 & 0.19 & 0 & 0.21 & 0.16 & 0 \\
\midrule
DoReMi & 0.08 & 0.11 & 0.18 & 0.05 & 0.08 & 0.14 & 0.04 & 0.16 & 0.13 \\
\bottomrule
\end{tabular}
\end{table}
Of particular interest is that \JointSLBO{} optimizes the data mixture by placing greater weights on some data domains based on their evaluation performance on the downstream task (in this case, GSM$8$K). Specifically, \JointSLBO{} successfully inferred (without knowing that the evaluation task is GSM$8$K) that domains such as GSM$8$K, SciQ, TriviaQA, Wikipedia, and MMLU contain some math information and therefore includes them in the optimized data mixture.

On the other hand, DoReMi is a distributionally robust data mixing method and results in a more uniform data mixing ratio. As a result, the output data mixtures are not specialized for any single evaluation task, leading to lower performance when evaluated on specific individual tasks, such as GSM$8$K.
\begin{table}[h]
\centering
\caption{Optimal model configuration found by \JointSLBO{} vs.~other baselines.} \label{table-optimal-model-config}
\begin{tabular}{c *{10}{c}}
\toprule
 & Rank & NumLayers & Order & Q & K & V & Up & Down & dropout & $\alpha$ \\
\midrule
\JointSLBO{} & 36 & 25 & 1 & 1 & 0 & 1 & 1 & 0 & 0.112 & 64\\
\midrule
DARTS & 12 & 13 & 0 & 1 & 1 & 1 & 1 & 0 & 0.058 & 45 \\
\bottomrule
\end{tabular}
\end{table}

Next, we examine the optimal model configurations found by \JointSLBO{}. We noticed that \JointSLBO{} prefers a higher LoRA rank and number of layers (i.e., how many layers to apply LoRA), but chooses to apply LoRA to only certain transformer layers. In particular, \JointSLBO{} found that for the GSM$8$K evaluation task, fine-tuning $Q$-, $V$-, and MLP-up-projection layers is sufficient to achieve good fine-tuning performance, and we should fine-tune the rear layers instead of the front layers (Order $= 1$).

\section{Additional experimental results and discussion}
\label{app:more-exp-results}

\subsection{Additional experiments against data and model optimization methods}
\label{app:more-data-model-exps}

Tables~\ref{table:table-mix-and-match-truthful}, \ref{table:mix-and-match-additional-commonsense}, \ref{table:mix-and-match-additional-MMLU}, \ref{table:mix-and-match-additional-ARC}, and~\ref{table:mix-and-match-additional-TriviaQA} extend the evaluation from Table~\ref{table:table-mix-and-match-gsm8k} to five other language tasks: TruthfulQA, CommonsenseQA, MMLU, ARC, and TriviaQA. \JointSLBO{} consistently outperforms all combinations of data and model optimization methods confirming that co-optimizing data and model configurations yields meaningful gains over treating them independently.

\begin{table}[h!]
  \centering
  \caption{{TruthfulQA} \citep{truthfulQA}.}
    \resizebox{\textwidth}{!}{%
    \begin{tabular}{llcccccccc}
      \toprule
      \multicolumn{2}{c}{\bfseries $\downarrow$ Model | Data $\rightarrow$}
        & Default & LESS & DoReMi & IF & Diversity  & BO & \JointSLBO{} \\
      \midrule
      \multirow{4}{*}{}
        & Default  & $59.8_{\pm 1.9}$ & $60.7_{\pm1.0}$ & $62.9_{\pm2.6}$ & $62.3_{\pm 1.3}$ & $64.5_{\pm 1.2}$            & $74.8_{\pm1.0}$ & -\\
        & DARTS        & $61.9_{\pm1.3}$ & $62.4_{\pm0.6}$ &  $66.7_{\pm1.3}$ &  $63.8_{\pm0.5}$ & $64.8_{\pm1.2}$ &  $77.0_{\pm0.9}$ & -   \\
        & AutoLoRA        &  $61.7_{\pm1.0}$   & $66.9_{\pm1.2}$ & $64.5_{\pm1.1}$  &   $65.3_{\pm1.2}$   & $65.9_{\pm0.6}$  & $72.5_{\pm0.5}$ & - \\
        & RoBoT   & $64.2_{\pm0.6}$  & $65.8_{\pm0.7}$ & $58.9_{\pm1.3}$ &  $56.9_{\pm1.0}$ & $65.8_{\pm0.6}$ &  $73.9_{\pm1.3}$ &   -      \\
        & BO   & $66.8_{\pm1.2}$  & $67.9_{\pm0.5}$ & $69.8_{\pm0.9}$ & $70.9_{\pm1.0}$ & $64.7_{\pm1.4}$ & $76.1_{\pm2.0}$ & -     \\
        & \JointSLBO{}   & - & -  & - & - & -&  -& $\cellcolor{darkgreen!50}\mathbf{\underline{78.0}}_{\pm0.0}$    \\
      \bottomrule
    \end{tabular}%
  }
  \label{table:table-mix-and-match-truthful}
\end{table}

\begin{table}[h!]
  \centering
  \caption{CommonsenseQA \citep{talmor2019commonsenseqaquestionansweringchallenge}.}
  \resizebox{\textwidth}{!}{%
    \begin{tabular}{llcccccccc}
      \toprule
      \multicolumn{2}{c}{\bfseries $\downarrow$ Model | Data $\rightarrow$}
        & Default & LESS & DoReMi & IF & Diversity  & BO & \JointSLBO{} \\
      \midrule
      \multirow{4}{*}{}
        & Default  & $76.3_{\pm 1.0}$ & $73.0_{\pm0.8}$ & $74.2_{\pm1.7}$ & $79.3_{\pm 0.7}$ & $77.4_{\pm 1.7}$ & $80.6_{\pm0.8}$ & - \\
        & DARTS        & $79.6_{\pm1.3}$ & $76.3_{\pm1.7}$ &  $76.1_{\pm1.1}$ &  $73.7_{\pm1.2}$ & $80.1_{\pm1.1}$ &  $79.6_{\pm0.6}$ & -   \\
        & AutoLoRA        &  $78.9_{\pm0.9}$   & $79.8_{\pm0.4}$ & $76.1_{\pm0.5}$  &   $77.9_{\pm1.2}$   & $78.0_{\pm1.0}$  & $81.5_{\pm1.0}$ & - \\
        & RoBoT   & $74.9_{\pm0.8}$  & $75.5_{\pm0.9}$ & $77.1_{\pm0.9}$ &  $79.4_{\pm1.5}$ & $76.3_{\pm0.9}$ &  $80.2_{\pm0.2}$ & -      \\
        & BO   & $79.7_{\pm1.3}$  & $79.4_{\pm0.3}$ & $77.0_{\pm0.4}$ & $81.1_{\pm0.9}$ & $79.4_{\pm1.1}$ & $80.7_{\pm1.2}$ & -     \\
        & \JointSLBO{}   & - &  - & - & - & -&  -& $\cellcolor{darkgreen!50}\mathbf{\underline{88.0}}{\pm1.5}$       \\
      \bottomrule
    \end{tabular}%
  }
  \label{table:mix-and-match-additional-commonsense}
\end{table}

\begin{table}[h!]
  \centering
  \caption{MMLU \citep{hendrycks2021measuringmassivemultitasklanguage}.}
  \resizebox{\textwidth}{!}{%
    \begin{tabular}{llcccccccc}
      \toprule
      \multicolumn{2}{c}{\bfseries $\downarrow$ Model | Data $\rightarrow$}
        & Default & LESS & DoReMi & IF & Diversity  & BO & \JointSLBO{} \\
      \midrule
      \multirow{4}{*}{}
        & Default  & $61.2_{\pm1.3}$ & $63.5_{\pm0.9}$ & $59.7_{\pm1.8}$ & $57.9_{\pm0.6}$ & $62.1_{\pm1.4}$ & $64.2_{\pm1.2}$ & - \\
        & DARTS        & $58.3_{\pm0.7}$ & $61.0_{\pm2.1}$ & $62.9_{\pm1.0}$ & $55.7_{\pm1.6}$ & $60.1_{\pm0.5}$ & $63.4_{\pm2.0}$ & - \\
        & AutoLoRA     & $62.5_{\pm1.4}$ & $64.3_{\pm0.6}$ & $60.8_{\pm2.2}$ & $58.2_{\pm1.9}$ & $63.7_{\pm0.8}$ & $61.5_{\pm1.1}$ & - \\
        & RoBoT        & $59.9_{\pm0.9}$ & $60.7_{\pm1.2}$ & $63.4_{\pm1.7}$ & $61.5_{\pm1.5}$ & $58.3_{\pm2.3}$ & $62.1_{\pm0.7}$ & - \\
        & BO           & $55.8_{\pm1.8}$ & $57.2_{\pm0.4}$ & $61.3_{\pm1.2}$ & $60.5_{\pm1.6}$ & $63.9_{\pm1.0}$ & $59.6_{\pm1.5}$ & - \\
        & \JointSLBO{} & - & - & - & - & - & - & $\cellcolor{darkgreen!50}\mathbf{\underline{74.4}}_{\pm0.4}$ \\
      \bottomrule
    \end{tabular}%
  }
  \label{table:mix-and-match-additional-MMLU}
\end{table}

\begin{table}[h!]
  \centering
  \caption{ARC \citep{ai2_arc}.}
  \resizebox{\textwidth}{!}{%
    \begin{tabular}{llcccccccc}
      \toprule
      \multicolumn{2}{c}{\bfseries $\downarrow$ Model | Data $\rightarrow$}
        & Default & LESS & DoReMi & IF & Diversity  & BO & \JointSLBO{} \\
      \midrule
      \multirow{4}{*}{}
        & Default  & $54.7_{\pm1.3}$ & $59.2_{\pm0.7}$ & $61.4_{\pm2.0}$ & $52.8_{\pm1.5}$ & $60.6_{\pm0.9}$ & $62.3_{\pm1.2}$ & - \\
        & DARTS        & $58.1_{\pm0.8}$ & $61.0_{\pm1.6}$ & $62.8_{\pm0.5}$ & $54.3_{\pm2.1}$ & $57.9_{\pm1.7}$ & $60.5_{\pm0.6}$ & - \\
        & AutoLoRA     & $60.4_{\pm1.1}$ & $63.2_{\pm0.9}$ & $58.6_{\pm1.9}$ & $55.1_{\pm0.8}$ & $62.1_{\pm2.0}$ & $59.8_{\pm1.0}$ & - \\
        & RoBoT        & $56.8_{\pm1.5}$ & $58.7_{\pm1.4}$ & $61.1_{\pm1.2}$ & $60.3_{\pm0.7}$ & $55.7_{\pm2.2}$ & $61.4_{\pm1.3}$ & - \\
        & BO           & $52.6_{\pm2.0}$ & $55.9_{\pm0.6}$ & $59.7_{\pm1.3}$ & $58.5_{\pm1.4}$ & $63.4_{\pm1.0}$ & $57.4_{\pm0.9}$ & - \\
        & \JointSLBO{} & - & - & - & - & - & - & $\cellcolor{darkgreen!50}\mathbf{\underline{65.6}}_{\pm0.0}$ \\
      \bottomrule
    \end{tabular}%
  }
  \label{table:mix-and-match-additional-ARC}
\end{table}

\begin{table}[h!]
  \centering
  \caption{TriviaQA Gen \citep{triviaQA}.}
  \resizebox{\textwidth}{!}{%
    \begin{tabular}{llcccccccc}
      \toprule
      \multicolumn{2}{c}{\bfseries $\downarrow$ Model | Data $\rightarrow$}
        & Default & LESS & DoReMi & IF & Diversity  & BO & \JointSLBO{} \\
      \midrule
      \multirow{4}{*}{}
        & Default  & $55.5_{\pm 1.4}$ & $57.2_{\pm0.8}$ & $53.1_{\pm0.9}$ & $55.8_{\pm 0.7}$ & $58.9_{\pm 0.8}$ & $65.0_{\pm0.6}$ & - \\
        & DARTS        & $58.2_{\pm0.8}$ & $61.3_{\pm1.2}$ & $61.0_{\pm0.7}$ & $63.3_{\pm1.0}$ & $59.2_{\pm0.6}$ & $66.7_{\pm1.8}$ & - \\
        & AutoLoRA     & $67.8_{\pm1.4}$ & $64.7_{\pm0.9}$ & $70.6_{\pm2.2}$ & $68.6_{\pm1.7}$ & $66.2_{\pm1.5}$ & $69.7_{\pm2.4}$ & - \\
        & RoBoT        & $58.4_{\pm1.5}$ & $62.3_{\pm1.7}$ & $64.2_{\pm1.4}$ & $57.2_{\pm1.2}$ & $63.4_{\pm1.5}$ & $68.2_{\pm1.3}$ & - \\
        & BO           & $70.7_{\pm1.4}$ & $66.7_{\pm0.8}$ & $72.5_{\pm0.8}$ & $71.7_{\pm0.9}$ & $74.7_{\pm1.0}$ & $72.7_{\pm2.3}$ & - \\
        & \JointSLBO{} & - & - & - & - & - & - & $\cellcolor{darkgreen!50}\mathbf{\underline{72.1}}{\pm0.4}$ \\
      \bottomrule
    \end{tabular}%
  }
  \label{table:mix-and-match-additional-TriviaQA}
\end{table}

\subsection{Extension to different model sizes, full-parameter fine-tuning and evaluation loss}
\label{app:extensions}
We extend our experiments to larger models (\texttt{Qwen3-14B}) (Fig.~\ref{fig:results-qwen-performance}), full-parameter fine-tuning (Table~\ref{table:rebuttal-performance-full-finetuning-larger-model}), and evaluation loss as the downstream metric (Fig.~\ref{fig:results-llama-loss}). In the full-parameter fine-tuning setting, LoRA is not used; instead, \JointSLBO{} optimizes which LLM layers to fine-tune directly. Our results show that \JointSLBO{} remains consistently effective across different model families, model sizes, training settings, and downstream metrics.

\begin{figure}[h]
\centering
\includegraphics[width=\linewidth]{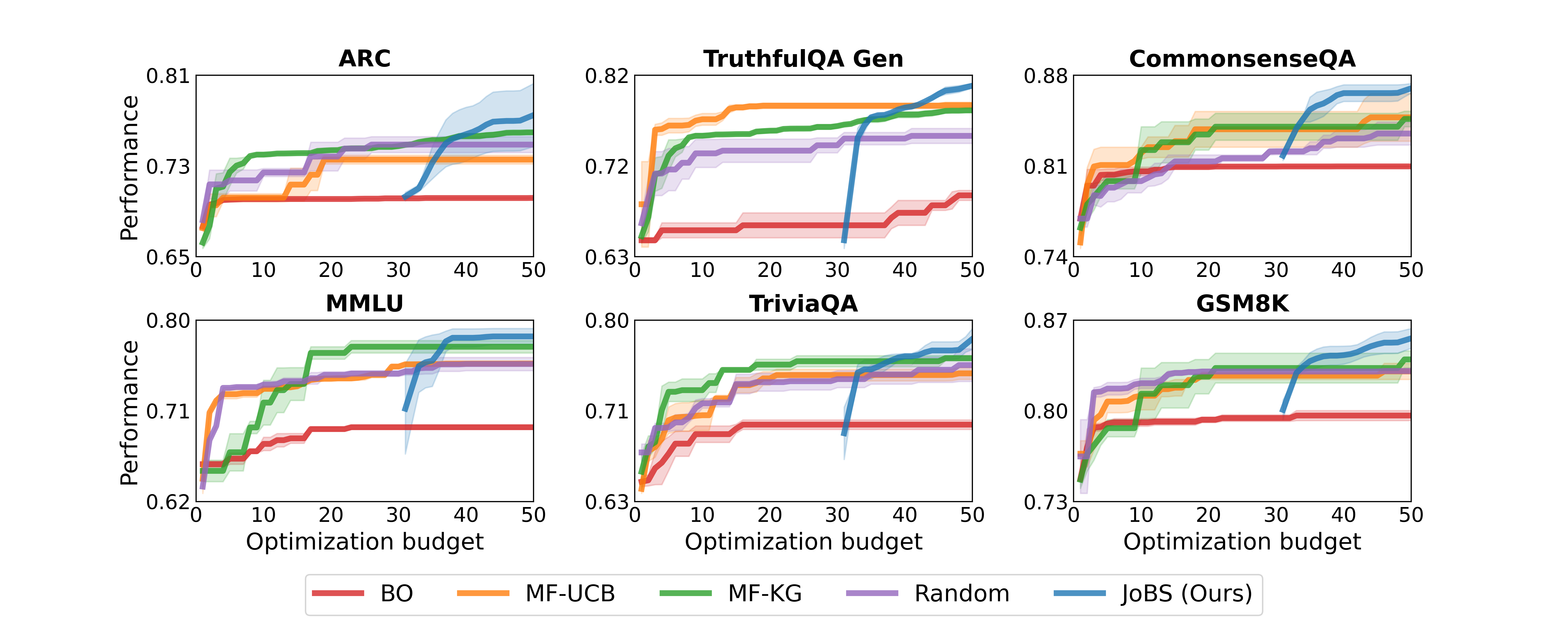}
\caption{\JointSLBO{} performance for LLM evaluation performance with \texttt{Qwen-14B}.}
\label{fig:results-qwen-performance}
\end{figure}

\begin{table}[!htbp]
\caption{\JointSLBO{} applied to full-parameter fine-tuning of larger models.}
\label{table:rebuttal-performance-full-finetuning-larger-model}
\centering
\begin{tabular}{|c|c|c|c|}
\hline
$\downarrow$ Model | Method $\rightarrow$ & LESS + AutoLoRA & DoReMi + DARTS & \JointSLBO{} \\ \hline
Llama-3-8B-Instruct & 0.73 & 0.76 & \textbf{0.81}  \\
Qwen3-14B & 0.76 & 0.81 & \textbf{0.83} \\
Qwen3-32B & 0.86 & 0.82 & \textbf{0.88} \\ \hline
\end{tabular}
\end{table}

\begin{figure}[h]
\centering
\includegraphics[width=\linewidth]{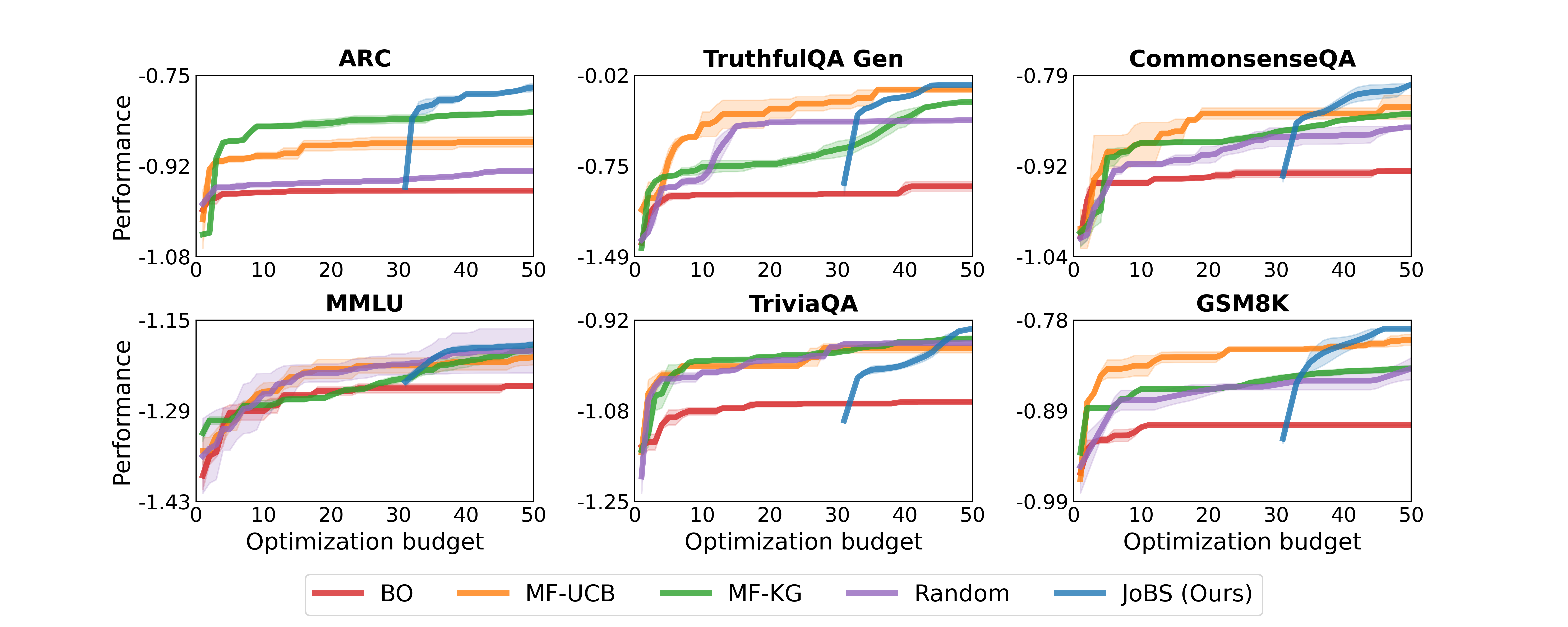}
\caption{\JointSLBO{} performance for LLM evaluation loss (instead of performance).}
\label{fig:results-llama-loss}
\end{figure}

\subsection{Additional BO baselines} 
\label{app:random-bo-baselines}
In Table~\ref{table-sanity}, we jointly optimized LLM training configurations using several other naive approaches in our experiments. We experimented with $3$ naive approaches: (a) \textbf{Random} randomly picks $500$ different LLM training configurations, fine-tunes the LLM for $100$ training steps each, use our performance scaling law predictor to predict and select the best-performing LLM training configuration. (b) \textbf{Random Data} performs \JointSLBO{} on model configurations for only $10$ iterations and repeats the experiment with $10$ randomly chosen data configurations (this ensures the same amount of compute as performing \JointSLBO{} on all LLM training configurations for $100$ iterations). (c) \textbf{Random Model} repeats approach (b) on LLM training configurations instead. While these approaches serve as good sanity checks, they do not yield good LLM performances, largely because randomly selecting LLM training configurations does not exploit the learnt performance landscape from historically observed LLM performances.

\begin{table}[h]
  \centering
  \caption{Comparison of some naive baselines with \JointSLBO{} (Higher is better), averaged over $5$ trials. \textbf{Random Data} means that we randomly selected data mixtures and applied \JointSLBO{} only on the model configurations (and vice versa for \textbf{Random Model}).}
  \resizebox{0.75\textwidth}{!}{%
    \begin{tabular}{llccc}
      \toprule
      \multicolumn{2}{l}{Model \hspace{31mm} Task}
        & \textbf{Random Data}  & \textbf{Random Model} & \JointSLBO{} \\
      \midrule
      \multirow{7}{*}{Llama-3-8B-Instruct}
        & GSM$8$K  & $67.3_{\pm1.6}$ &  $71.5_{\pm0.9}$ & \cellcolor{darkgreen!50}$86.4_{\pm{1.2}}$ \\
        & TruthfulQA        &  $59.8_{\pm1.5}$ & $64.2_{\pm1.4}$  & \cellcolor{darkgreen!50}$80.6_{\pm{1.1}}$   \\
        & CommonsenseQA        &  $76.4_{\pm1.2}$ &  $76.3_{\pm1.2}$  & \cellcolor{darkgreen!50}$84.3_{\pm{2.4}}$ \\
        & MMLU   & $66.4_{\pm0.7}$ & $63.1_{\pm1.1}$ & \cellcolor{darkgreen!50}$69.5_{\pm{0.8}}$        \\
        & ARC   & $65.2_{\pm1.7}$ & $64.6_{\pm0.6}$ & \cellcolor{darkgreen!50}$70.4_{\pm{1.3}}$        \\
        & TriviaQA   & $61.7_{\pm2.4}$ & $63.2_{\pm1.5}$ & \cellcolor{darkgreen!50}$76.2_{\pm{1.2}}$        \\
      \bottomrule
    \end{tabular}%
  }
  \label{table-sanity}
\end{table}

\newpage
\section{Asset licenses}
\label{app:licenses}

\begin{table}[h]
\centering
\small
\setlength{\extrarowheight}{2pt}
\begin{tabular}{>{\raggedright}p{0.15\textwidth} 
                >{\raggedright}p{0.15\textwidth} 
                >{\raggedright\arraybackslash}p{0.7\textwidth}}
\toprule
\textbf{Asset} & \textbf{License} & \textbf{URL} \\
\midrule
Wikitext & CC-BY-SA 3.0 & \url{https://huggingface.co/datasets/mindchain/wikitext2} \\
GSM8K & MIT & \url{https://huggingface.co/datasets/openai/gsm8k} \\
PubmedQA & MIT & \url{https://huggingface.co/datasets/qiaojin/PubMedQA} \\
SciQ & CC-BY-NC 3.0 & \url{https://huggingface.co/datasets/allenai/sciq} \\
TriviaQA & Apache 2.0 & \url{https://huggingface.co/datasets/mandarjoshi/trivia_qa} \\
TruthfulQA & Apache 2.0 & \url{https://huggingface.co/datasets/domenicrosati/TruthfulQA} \\
MMLU & MIT & \url{https://huggingface.co/datasets/cais/mmlu} \\
AI2 ARC & CC-BY-SA 4.0 & \url{https://huggingface.co/datasets/allenai/ai2_arc} \\
CommonsenseQA & MIT & \url{https://huggingface.co/datasets/tau/commonsense_qa} \\
Llama-3-8B-Instruct & Llama 3 Community License & \url{https://huggingface.co/meta-llama/Meta-Llama-3-8B-Instruct} \\
Qwen3 & Apache 2.0 & \url{https://huggingface.co/collections/Qwen/qwen3} \\
lm-evaluation-harness & MIT & \url{https://github.com/EleutherAI/lm-evaluation-harness} \\
BoTorch & MIT & \url{https://github.com/meta-pytorch/botorch} \\
\bottomrule
\end{tabular}
\end{table}

\newpage
\section*{NeurIPS Paper Checklist}

\begin{enumerate}

\item {\bf Claims}
    \item[] Question: Do the main claims made in the abstract and introduction accurately reflect the paper's contributions and scope?
    \item[] Answer: \answerYes{} 
    \item[] Justification: The abstract and introduction accurately reflect the paper's contributions and scope. The main contributions can be found at the end of the introduction.

\item {\bf Limitations}
    \item[] Question: Does the paper discuss the limitations of the work performed by the authors?
    \item[] Answer: \answerYes{} 
    \item[] Justification: The paper discusses the limitations of the work in the conclusion section, as well as appendix.

\item {\bf Theory assumptions and proofs}
    \item[] Question: For each theoretical result, does the paper provide the full set of assumptions and a complete (and correct) proof?
    \item[] Answer: \answerYes{} 
    \item[] Justification: The paper provides assumptions and proofs in Sec.~\ref{sec:BO}, Sec.~\ref{subsec:convergence-theorem}, and App.~\ref{app:proof-BO-convergence}

    \item {\bf Experimental result reproducibility}
    \item[] Question: Does the paper fully disclose all the information needed to reproduce the main experimental results of the paper to the extent that it affects the main claims and/or conclusions of the paper (regardless of whether the code and data are provided or not)?
    \item[] Answer: \answerYes{} 
    \item[] Justification: We release our code at \url{https://github.com/a35453779/JoBS}.

\item {\bf Open access to data and code}
    \item[] Question: Does the paper provide open access to the data and code, with sufficient instructions to faithfully reproduce the main experimental results, as described in supplemental material?
    \item[] Answer: \answerYes{} 
    \item[] Justification: We release our code at \url{https://github.com/a35453779/JoBS}.

\item {\bf Experimental setting/details}
    \item[] Question: Does the paper specify all the training and test details (e.g., data splits, hyperparameters, how they were chosen, type of optimizer) necessary to understand the results?
    \item[] Answer: \answerYes{} 
    \item[] Justification: We specify experimental settings in Sec.~\ref{sec:exp} and App.~\ref{app:exp-details}.

\item {\bf Experiment statistical significance}
    \item[] Question: Does the paper report error bars suitably and correctly defined or other appropriate information about the statistical significance of the experiments?
    \item[] Answer: \answerYes{}
    \item[] Justification: We include error bars fill areas in the experimental results (Fig.~\ref{fig:BO-convergence}). We analyze our predictor error's distribution and theoretical effects on the BO process.

\item {\bf Experiments compute resources}
    \item[] Question: For each experiment, does the paper provide sufficient information on the computer resources (type of compute workers, memory, time of execution) needed to reproduce the experiments?
    \item[] Answer: \answerYes{} 
    \item[] Justification: We specify the compute resources used in App.~\ref{app:exp-details} and App.~\ref{app:compute-cost}.
    
\item {\bf Code of ethics}
    \item[] Question: Does the research conducted in the paper conform, in every respect, with the NeurIPS Code of Ethics \url{https://neurips.cc/public/EthicsGuidelines}?
    \item[] Answer: \answerYes{} 
    \item[] Justification: The research conforms with the NeurIPS Code of Ethics, and the authors have made sure to preserve anonymity.

\item {\bf Broader impacts}
    \item[] Question: Does the paper discuss both potential positive societal impacts and negative societal impacts of the work performed?
    \item[] Answer: \answerYes{}
    \item[] Justification: Impact statement is included in the appendix.
    
\item {\bf Safeguards}
    \item[] Question: Does the paper describe safeguards that have been put in place for responsible release of data or models that have a high risk for misuse (e.g., pre-trained language models, image generators, or scraped datasets)?
    \item[] Answer: \answerNA{} 
    \item[] Justification: The paper poses no such risks.

\item {\bf Licenses for existing assets}
    \item[] Question: Are the creators or original owners of assets (e.g., code, data, models), used in the paper, properly credited and are the license and terms of use explicitly mentioned and properly respected?
    \item[] Answer: \answerYes{}
    \item[] Justification: All datasets, models, and baseline implementations used in this paper are publicly available and properly cited. Specifically: Wikitext (CC-BY-SA 3.0), GSM8K (MIT), PubmedQA (MIT), SciQ (CC-BY-NC 3.0), TriviaQA (Apache 2.0), TruthfulQA (Apache 2.0), MMLU (MIT), AI2 ARC (CC-BY 4.0), and CommonsenseQA (MIT). Llama-3-8B-Instruct is used in accordance with the Meta Llama 3 Community License; Qwen models are used in accordance with the Qwen License. The lm-evaluation-harness framework is used under the MIT License. Baseline implementations follow the licenses of their respective repositories. A full summary of asset licenses and URLs is provided in App.~\ref{app:licenses}.

\item {\bf New assets}
    \item[] Question: Are new assets introduced in the paper well documented and is the documentation provided alongside the assets?
    \item[] Answer: \answerYes{}
    \item[] Justification: The code is fully open-sourced and available at the provided GitHub link.

\item {\bf Crowdsourcing and research with human subjects}
    \item[] Question: For crowdsourcing experiments and research with human subjects, does the paper include the full text of instructions given to participants and screenshots, if applicable, as well as details about compensation (if any)? 
    \item[] Answer: \answerNA{} 
    \item[] Justification: Our work does not involve crowdsourcing nor research with human subjects.

\item {\bf Institutional review board (IRB) approvals or equivalent for research with human subjects}
    \item[] Question: Does the paper describe potential risks incurred by study participants, whether such risks were disclosed to the subjects, and whether Institutional Review Board (IRB) approvals (or an equivalent approval/review based on the requirements of your country or institution) were obtained?
    \item[] Answer: \answerNA{} 
    \item[] Justification: Our work does not involve crowdsourcing nor research with human subjects.

\item {\bf Declaration of LLM usage}
    \item[] Question: Does the paper describe the usage of LLMs if it is an important, original, or non-standard component of the core methods in this research? Note that if the LLM is used only for writing, editing, or formatting purposes and does \emph{not} impact the core methodology, scientific rigor, or originality of the research, declaration is not required.
    \item[] Answer: \answerNA{} 
    \item[] Justification: The core method development in our work does not involve LLMs as any important, original, or non-standard components.

\end{enumerate}

\end{document}